\documentclass[journal]{IEEEtran}

\usepackage{bibentry}  
\usepackage{xcolor,soul,framed} 
\colorlet{shadecolor}{yellow}
\usepackage[pdftex]{graphicx}
\graphicspath{{../pdf/}{../jpeg/}}
\DeclareGraphicsExtensions{.pdf,.jpeg,.png}
\usepackage[cmex10]{amsmath}
\usepackage{array}
\usepackage{mdwmath}
\usepackage{mdwtab}
\usepackage{eqparbox}
\usepackage{url}
\usepackage{array}
\usepackage{amsmath}
\usepackage{amsthm}
\usepackage{amssymb}
\usepackage{graphicx}
\usepackage{caption}
\usepackage{wrapfig}
\usepackage[algo2e]{algorithm2e}
\usepackage{algorithm}
\usepackage{algpseudocode}
\usepackage{float}
\usepackage{subfig}
\usepackage{dblfloatfix}
\usepackage{wrapfig}
\usepackage{setspace}
\usepackage{pifont}
\usepackage{tabularx, multirow, booktabs}
\DeclareMathOperator{\supp}{supp}
\DeclareMathOperator{\rank}{rank}
\DeclareMathOperator{\size}{size}
\newtheorem{theorem}{Theorem}
\newtheorem{corollary}{Corollary}
\newtheorem{lemma}{Lemma}
\newtheorem{remark}{Remark}
\newcommand{\ra}[1]{\renewcommand{\arraystretch}{#1}}

\begin{document}
\onecolumn

    \title{Compressive Sensing Using Iterative Hard Thresholding with Low Precision Data Representation: Theory and Applications}
  \author{Nezihe Merve G\"urel,
      Kaan Kara,
      Alen~Stojanov,
      Tyler~Smith,
      Thomas Lemmin,\\
      Dan~Alistarh,
      Markus~P\"uschel,
      and~Ce~Zhang

  \thanks{Manuscript received September 12, 2019; revised April 29, 2020; accepted July 3, 2020. The associate editor coordinating the review of this manuscript and approving it for publication was Dr. Weiyu Xu. (Corresponding author: Nezihe Merve G\"urel).}
  \thanks{1053-587X $\copyright$ 2020 IEEE. Permission from IEEE must be obtained for all other uses, in any current or future media, including reprinting/republishing this material for advertising or promotional purposes, creating new collective works, for resale or redistribution to servers or lists, or reuse of any copyrighted component of this work in other works.
  Digital Object Identifier: 10.1109/TSP.2020.3010355}
  \thanks{N. M. G\"urel, K. Kara, A. Stojanov, T. Smith, T. Lemmin, M. P\"uschel and C. Zhang are with the Department of Computer Science, ETH Zurich, CH-8092 Zurich, Switzerland (e-mails: \{nezihe.guerel, k.kara, astojanov, tyler.smith, thomas.lemmin, pueschel, ce.zhang\}@inf.ethz.ch). }
  \thanks{D. Alistarh is with the Institute of Science and Technology Austria, A-3400 Klosterneuburg, Austria (e-mail: dan.alistarh@ist.ac.at).}} 

\markboth{IEEE TRANSACTIONS ON SIGNAL PROCESSING~Vol. 68, No. 7, pp. 4268-4282, 2020
}{Roberg \MakeLowercase{\textit{et al.}}: Compressive Sensing Using Iterative Hard Thresholding with Low Precision Data Representation: Theory and Applications}

\maketitle

\begin{abstract}
Modern scientific instruments produce vast amounts of data, which can overwhelm the processing ability of computer systems. Lossy compression of data is an intriguing solution, but comes with its own drawbacks, such as potential signal loss, and the need for careful optimization of the compression ratio. In this work, we focus on a setting where this problem is especially acute: compressive sensing frameworks for interferometry and medical imaging. We ask the following question: can the precision of the data representation be lowered for all inputs, with recovery guarantees and practical performance? Our first contribution is a theoretical analysis of the normalized Iterative Hard Thresholding (IHT) algorithm when all input data, meaning both the measurement matrix and the observation vector are quantized aggressively. We present a variant of low precision normalized {IHT} that, under mild conditions, can still provide recovery guarantees. The second contribution is the application of our quantization framework to radio astronomy and magnetic resonance imaging. We show that lowering the precision of the data can significantly accelerate image recovery. We evaluate our approach on telescope data and samples 
of brain images using CPU and FPGA implementations achieving up to a 9x speed-up with negligible loss of recovery quality.
\end{abstract}


\IEEEpeerreviewmaketitle

\section{Introduction}\label{Sec:Introduction}
\IEEEPARstart{T}{he} ability to collect, store, and process substantial amounts of data is enabling a next generation of {\em data intensive} scientific instruments. Such instruments require extremely advanced capabilities in terms of engineering, algorithms, calibration, and storage~\cite{Hu2014}. Compressive sensing~\cite{donoho2006cs, candes2006cs, candes2006cs2} is a powerful mathematical framework behind many of these instruments. Compressive sensing algorithms can learn the sparse representation of analog signals from only a few samples, enabling the efficient collection, processing, and storage of very large amounts of data.
An interesting property of the compressive sensing problem
and of many compressive sensing solvers is
their tolerance to noise introduced by
{\em data quantization}. Several previous studies
have taken advantage of this, decreasing the
precision of data representation to as low as a single bit~\cite{boufounos20091bitcs, ai20121bitcs, jacques20111bit, laska20111bitcs,plan20111bitcs,plan20121bitcs, gupta2015dl, gopi20131bitcs}.

\begin{table}[t]
\centering
\ra{0.9}
\caption{Comparison with previous work. $Q({\bf \Phi})$ and $Q({\bf y})$ indicate whether quantization of the measurement matrix ${\bf \Phi}$ or quantization of the observation vector ${\bf y}$ are considered
($\checkmark$: yes,  {\scriptsize \ding{53}} : no).}
\label{tab:cs}
\begin{tabular}{@{}llcc@{}}\toprule
                  &{Assumption on ${\bf \Phi}$ }& {$Q(\bf{\Phi})$}  & {$Q({\bf y})$} \\
\hline\\ [-1.5ex]
{ Boufounos et al.~\cite{boufounos20091bitcs}}
 & {Gaussian}  & \ {\ding{53}} & \ $ \checkmark$\\
 { Ai et al.~\cite{ai20121bitcs}}        & {unit variance} & \ {\ding{53}} & \ $\checkmark$\\
 { Jacques et al.~\cite{jacques20111bit}}     &  {RIP} & \ {\ding{53}}& \ $ \checkmark$\\
 { Laska et al.~\cite{laska20111bitcs}}     & {Gaussian} \&  {RIP} &\ {\ding{53}} & \ $\checkmark$\\
 { Plan et al.~\cite{plan20111bitcs}}     & {Gaussian} \&  {RIP} & \ {\ding{53}} & \ $ \checkmark$\\
 { Plan et al.~\cite{plan20121bitcs}}      & {Gaussian} & \ {\ding{53}}& \ $\checkmark$\\
 { Gupta et al.~\cite{gupta2015dl}}     & {Gaussian} &\ {\ding{53}} & \ $ \checkmark$\\
 { Gopi et al.~\cite{gopi20131bitcs}}       & {sub-Gaussian/binary} \&  {RIP} & \ $\checkmark$ & \ $\checkmark$\\
  { {\bf  This work}  }              & {non-symmetric}  {RIP} & \ $\small \checkmark$ & \  $ \checkmark$ \\
\bottomrule
\end{tabular}
\end{table}

Most of the previous work focused on the case where quantization is carried out only on the observation vector (Table~\ref{tab:cs}). In only one single previous study, both the observation vector and the measurement matrix were quantized by imposing additional assumptions on the measurement matrix (sub-Gaussian or binary) ~\cite{gopi20131bitcs}.
In this paper, we take this direction further and investigate the design of a compressive
sensing solver which quantizes both the measurement matrix {\em and} the observation vector, while imposing a more general set of assumptions
on the measurement matrix.

\paragraph*{Summary of technical contributions}
The main technical contribution of this paper is
a new theoretical analysis  showing that normalized Iterative Hard Thresholding 
(IHT)~\cite{blumensath2010niht}, a popular
algorithm for compressive sensing,
converges with guarantees on the recovery
quality even when both the measurement matrix and the
observation vector are stored in lower
precision. This result holds provided that the measurement matrix
satisfies a mild Restricted Isometry Property (RIP) condition, known as non-symmetric RIP~\cite{blumensath2010niht, eldar2012cs}.

We validate our theoretical results in the context of  two real-world applications: 
radio astronomy and magnetic resonance imaging.
We show that our framework has strong signal recovery performance, as illustrated 
in Figure~\ref{fig:result_highlight}, by leveraging the structure of the measurement matrix.

Further, we implemented our approach on both {CPU} and {FPGA} platforms, demonstrating speedups of up to 7x and 9x for full recovery, respectively, 
on instances with a quantized dense measurement matrix.
We believe that the tools we developed are general enough to extend to other sparse reconstruction problems that demand high processing capability.

\paragraph*{Notation} Hereafter, scalars will be written in italics, vectors in bold lower-case and matrices in bold upper-case letters. We define ${\bf x}$ as an $N$-dimensional real or complex sparse vector and ${\bf y}$ as an $M$-dimensional real or complex observation vector. For an $M\times N$ real or complex measurement matrix ${\bf \Phi}$, the matrix element in the $m$th row and $n$th column is denoted as ${\bf \Phi}_{m, n}$ and its Hermitian transpose as ${\bf \Phi}^T$. Also, $\pmb{\phi}_n$ is the $n$th column of $\boldsymbol{\Phi}$ such that ${\bf \Phi}=[\pmb{\phi}_n]_{n=\{1, 2, \dots, N\}}$. The submatrix of ${\bf \Phi}$ obtained by selecting the columns with indices in $\Gamma$ is written as ${\bf \Phi}_{\Gamma} = [\pmb{\phi}_n]_{n\in \Gamma}$, and the $p$-norm by $\| \cdot \|_p$. For the sake of simplicity, we drop $p$ whenever $p=2$. Finally, a 32-bit representation is used for the full precision scheme and {b}${}_{\bf \Phi}$/{b}${}_{\bf y}$ denotes the number of bits used to represent the elements of the measurement matrix ${\bf \Phi}$ and the observation vector ${\bf y}$, respectively.

\subsection{Background and Problem Definition}\label{Sec:ProblemFormulation}

Compressive sensing~\cite{donoho2006cs, candes2006cs, candes2006cs2} is a technique in sparse signal reconstruction that offers a range of efficient algorithms acquiring high dimensional signals from inaccurate and incomplete samples with an underlying sparse structure. 
Many real-world applications including medical imaging, interferometry, and genomic data analysis benefit from these techniques.

In mathematical terms, compressive sensing is formulated as follows: Let a sparse or approximately sparse signal ${\bf x}$ be sampled via a linear sampling operator ${\bf \Phi}$. This means that the observation vector ${\bf y}$ is
\begin{equation}\label{CS_model}
 {\bf y} =  {\bf \Phi}  {\bf x} + {\bf e},
\end{equation}
where ${\bf e}$ is $M$-dimensional observation noise.

Compressive sensing recovery algorithms iteratively compute a sparse estimate $\tilde{{\bf x}}$ with $N\gg M$ such that ${\bf \Phi}  \tilde{\bf x}$ approximates ${\bf y}$ well, that is,  $\| {\bf y}- {\bf \Phi}  \tilde{\bf x} \|$ is small. This problem is
NP-hard due to its combinatorial nature. 
Therefore, most compressive sensing algorithms resort to a convex relaxation of the underlying sparse optimization problem. A collection of thresholding and greedy methods solving this problem have been proposed including Iterative Hard Thresholding ({IHT}) \cite{blumensath2008iht, blumensath2009iht}, Compressive Sampling Matching Pursuit ({CoSaMP}) \cite{needel2008cosamp}, as well as others others~\cite{liu2017dualiht, yuan2014ht, yuan2016htp, blumensath2013cs, needel2008cosamp}. These references also present a comprehensive analysis of the provable performance guarantees for such sparsity-constrained minimization methods, in terms of convergence to fixed point of $\ell_0$-regularized cost functions and the optimality of such approximations. However, when applied to real-life problems this prior work faces additional challenges. For provable guarantees, it is often required that (a) the measurement matrix ${\bf \Phi}$ satisfies the Restricted Isometry Property ({RIP})~\cite{candes2008rip, chartrand2008rip}, and that (b) the sparsity level is chosen appropriately. 
The Normalized {IHT} method~\cite{blumensath2010niht}, relaxes the {RIP} condition by introducing a step size parameter, which enables rigorous guarantees for a broader class of practical problems. Our paper builds upon this line of work.

\begin{figure}[t]
\centering
    \includegraphics[width=0.6\textwidth]{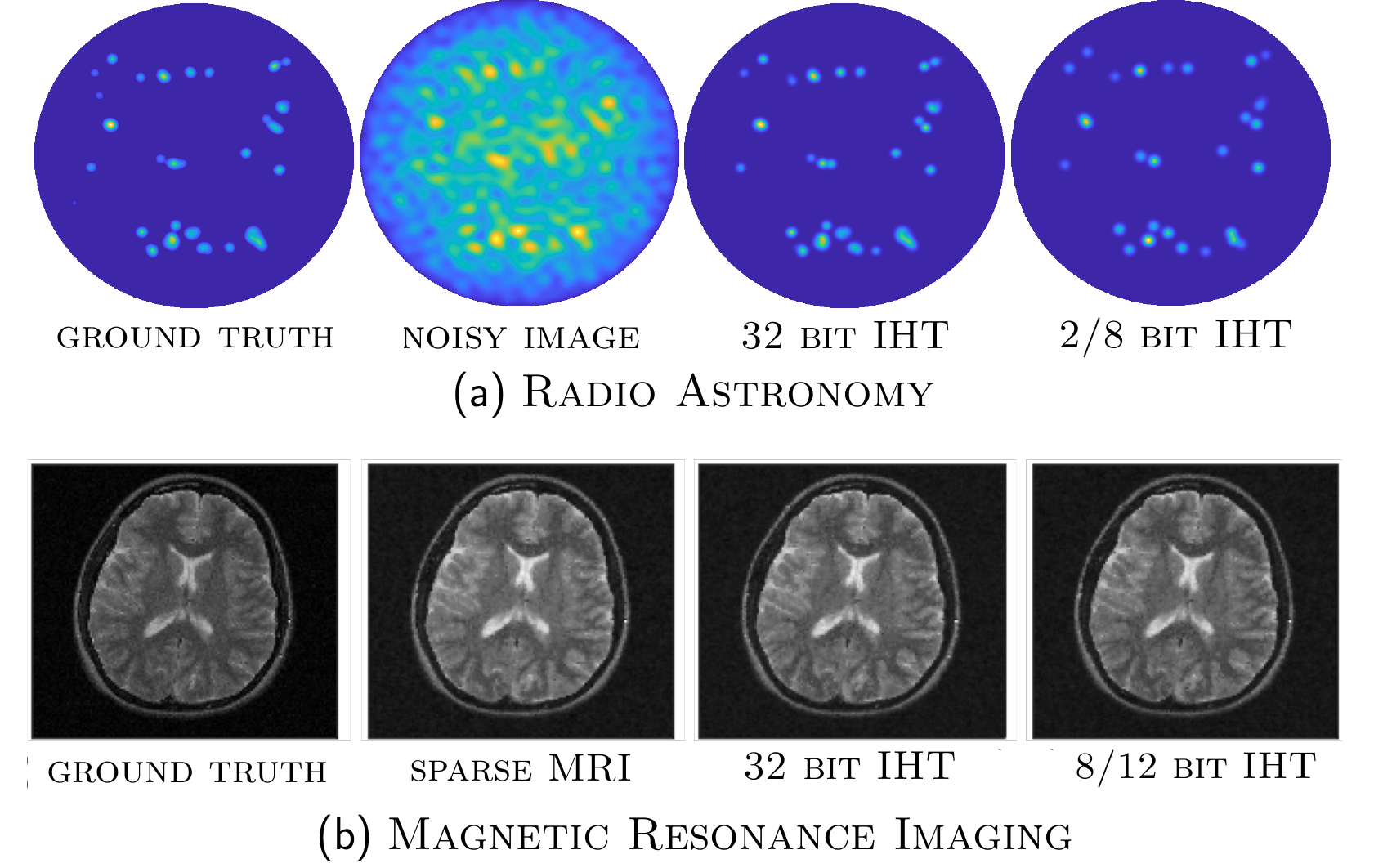}
  \caption{Illustration of the main results for the radio astronomy and the magnetic resonance imaging applications. When representing all input data with low precision, {IHT} achieves a negligible loss of recovery quality on the data recorded by (a) {LOFAR} station {CS302} with {2}-bit measurement matrix and {8}-bit observation, (b) subsampling $k$-space measurements (the 2D Fourier transform) of MRI images with {8}-bit measurement matrix and {12}-bit observation.}\label{summary_of_results}
\label{fig:result_highlight}
\end{figure}

We consider the sparse signal recovery problem in (\ref{CS_model}) described as: given ${\bf y}$ and ${\bf \Phi}$, find ${\bf x}$ minimizing the cost function
 \begin{equation}\label{cost_function}
 \small
     \|{\bf y} - {\bf \Phi x} \|^2 \ \ \rm{subject \ to \ \|{\bf x} \|_0 \leq s},
 \end{equation}
where $\|{\bf x} \|_0 = |{\supp}({\bf x})| = |\{i: x_i\neq 0 \}|$ and $s$ is number of sparse coefficients we want to recover.

\paragraph*{Normalized IHT}
Normalized IHT~\cite{blumensath2010niht} is an iterative solver of the optimization problem in (\ref{cost_function}) that is shown to outperform other methods such as traditional IHT and CoSaMP when the non-symmetric RIP condition holds. It uses the following update rule:
\begin{equation}\label{Eq:NIHT}
     {\bf x}^{[n+1]} = H_s({\bf x}^{[n]} + \mu^{[n]}{\bf \Phi}^{T}({\bf y} - {\bf \Phi}{\bf x}^{[n]})),
\end{equation}
where ${\bf x}^{[0]} = 0$ and $\mu^{[n]}>0$ is the adaptive step size parameter, $H_s({\bf x})$ is a nonlinear operator preserving only the largest $s$ entries of {\bf x} in magnitude, setting the other entries to zero.

If ${\bf x}$ has no more than $s$ nonzero elements,
the proposed update rule converges to a 
local minimum of the cost function $\|{\bf y} - {\bf \Phi x} \|^2$.
Furthermore, if the measurement matrix $\boldsymbol{\Phi}$ satisfies the non-symmetric RIP condition, normalized IHT is guarantees stability and performance,i.e., the result is near-optimal. The properties of normalized IHT are discussed in greater detail in Section~\ref{Sec:Prelimiaries}.

\paragraph*{Our Setting}
In this paper, we consider the properties of the
normalized IHT algorithm in a lossy compression setting, where both the data ${\bf y}$ and ${\bf \Phi}$ consist of floating-point values and undergo a stochastic quantization process to a small set of discrete levels, using a transformation operator. We denote the transformation operator by $Q(\cdot, b)$ where $b$ is the bit precision used by the representation. 
The goal of applying $Q(\cdot, b)$ is to reduce the high cost of data transmission between the sensor or storage and the computational device ({CPU}, {GPU}, or {FPGA}). We thus want to recover ${\bf x}$ using the  modified normalized {IHT} update rule
 \begin{equation*}
     {\bf x}^{[n+1]} = H_s({\bf x}^{[n]} + \mu^{[n]}Q({\bf \Phi}, b_{\bf \Phi})^{T}(Q({\bf y}, b_{\bf y}) - Q({\bf \Phi}, b_{\bf \Phi}){\bf x}^{[n]})).
 \end{equation*} 

\subsection{Related work}\label{Sec:RelatedWork}
Several studies have applied quantization in compressive sensing problems (Table~\ref{tab:cs}). They explore binary measurements for sparse signal recovery under different assumptions on the measurement matrix. Sparse signal recovery with a scale factor when measurements preserve only sign information was demonstrated in~\cite{boufounos20091bitcs}. Further, approximately sparse signals can be robustly recovered from single-bit measurements when sampled with a sub-Gaussian distribution~\cite{ai20121bitcs, davenport20121bit}. A similar setting is studied in~\cite{jacques20111bit, laska20111bitcs} with a Gaussian measurement matrix (Binary {IHT})~\cite{plan20111bitcs, plan20121bitcs}, which proposes a computationally tractable and optimal recovery of a 1-bit compressive sensing problem. The theoretical guarantees to recover the \textit{support} of high-dimensional sparse signals from 1-bit measurements are provided by~\cite{recht20121bitcs, gopi20131bitcs}. 

This paper differs from prior work in two main ways. First, our assumption that the measurement matrix is non-symmetric RIP is critical in real-life applications, 
and none of the assumptions made in prior work would fit this use case. Second, to the best of our knowledge, we are the only work besides~\cite{gopi20131bitcs} that quantizes
{\em both} the measurement matrix $\boldsymbol{\Phi}$ and the observation vector ${\bf y}$.
The problem of building a binary measurement matrix that can provide good recovery guarantees is considered in~\cite{gopi20131bitcs} given only one-bit measurements. By contrast, we consider a practical setting where we must quantize a given full-precision measurement matrix as well as possible, and thus can trade off higher precision for better recovery guarantees. 

Another emerging line of work has been on low precision training for 
machine learning applications beyond compressive
sensing. Examples include~\cite{desa2015hogwild, alistarh2016qsgd, zhang2017zipml} and
the line of work on 
partial or end-to-end low-precision training 
of deep networks~\cite{seide2014sgd1bit, hubara2016qsnn, rastegari2016binarycnn,zhou2016cnn, miyashita2016cnn, li2016twn, gupta2015dl}.
These references focus on quantization in the context of stochastic gradient descent ({SGD}). {IHT} and related methods can be seen as projected gradient methods, but existing results focus mostly on the variance added by quantization, and do not apply to recovery properties in the sparse case, which is our focus. 

There has been significant research on 
designing efficient algorithms for sparse recovery~\cite{blumensath2011aiht, wei2015fiht, blanchard2013iht, cevher2011ht, liu2017dualiht}.
We focus here on normalized {IHT}, and leave extensions to other methods as future work.
We further note the work on recovery using sparse binary matrices (see~\cite{gilbert2010sparse} for a survey). 
These matrix constructions could be applied in our scenario in some cases, as they are pre-quantized with similar guarantees. However, in  certain applications such as the radio astronomy and magnetic resonance imaging considered here, the measurement matrix is fixed and highly dense.

\section{Background: Normalized Iterative Hard Thresholding}
In this section, we review existing results on the normalized IHT algorithm~\cite{blumensath2010niht, blumensath2012greedy} (see Algorithm~\ref{Algorithm: Normalized_IHT}). These can be generalized to the traditional {IHT} if the measurement matrix satisfies $\|\boldsymbol{\Phi} \| < 1$~\cite{blumensath2009iht, blumensath2008iht}. 
\subsection{The Algorithm}\label{Sec:Prelimiaries}

\begin{figure}
\begin{algorithm}[H]
\SetAlgoLined
 {\bf Input:} Measurement matrix $\boldsymbol{\Phi}$, measurements ${\bf y}$, sparsity parameter $s$, number of iterations $n^*$, step size tuning parameters {\it k}, {\it c}
 
{\bf Output:} The recovery vector ${\bf x}^{[n^*]}$

{Initialize} ${\bf x}^{[0]} = 0$, $\Gamma^{[0]} = \supp\big (H_s({\bf \Phi}^{T}{{\bf y}})\big)$.

 \For{$n=1$ {\bf to} ${n}^*$}{
 ${\bf g}^{[n-1]} = {{\bf \Phi}^{T}} \big({{\bf y}}-{{\bf \Phi}}{\bf x}^{[n-1]}\big)$\;
 
 \vspace{0.15em}
 
 {\small $\boldsymbol{\mu}^{[n-1]} ={\big({\bf g}^{[n-1]}_{\Gamma^{[n-1]}}\big)^T{\bf g}^{[n-1]}_{\Gamma^{[n-1]}}}/{\big ( \boldsymbol{\Phi}_{\Gamma^{[n-1]}}{\bf g}^{[n-1]}_{\Gamma^{[n-1]}}\big)^T\boldsymbol{\Phi}_{\Gamma^{[n-1]}}{\bf g}^{[n-1]}_{\Gamma^{[n-1]}}}$}
 
 \vspace{0.25em}
 
 ${\bf x}^{[n]} = H_s({\bf x}^{[n-1]} + {\bf \mu}^{[n-1]} {\bf g}^{[n-1]})$
 
 \vspace{0.25em}
 
 $\Gamma^{[n]} = \supp({\bf x}^{[n]})$
 
 \vspace{0.15em}

\eIf{$\Gamma^{[n]}=\Gamma^{[n-1]}$}{

\vspace{0.15em}

${\bf x}^{[n]}={\bf x}^{[n-1]}$

}
{
$b^{[n]}=(\|{\bf x}^{[n]}-{\bf x}^{[n-1]} \|^2_2)/(\|\boldsymbol{\Phi}({\bf x}^{[n]}-{\bf x}^{[n-1]})\|^2_2)$

\vspace{0.15em}

\eIf{${\mu}^{[n]}\leq (1-{c})b^{[n]}$}{${\bf x}^{[n]}={\bf x}^{[n-1]}$}{

\vspace{0.25em}

\While{${\mu}^{[n]}> (1-{c})b^{[n]}$}{${\mu}^{[n]} = {\mu}^{[n]}/(k(1-c))$

\vspace{0.25em}

${\bf x}^{[n]} = H_s({\bf x}^{[n-1]} + \boldsymbol{\mu}^{[n-1]} {\bf g}^{[n-1]})$

\vspace{0.25em}

}
}
} 
 $\Gamma^{[n]} = \supp({\bf x}^{[n]})$\;
}
\caption{Normalized IHT}
\label{Algorithm: Normalized_IHT}
\end{algorithm}
\end{figure}

Let ${\bf x}^{[0]}=0$. As introduced in~(\ref{Eq:NIHT}), normalized IHT has the following update rule: 
$$
{\bf x}^{[n+1]} = H_s({\bf x}^{[n]} + \mu^{[n]}{\bf \Phi}^{T}({\bf y} - {\bf \Phi}{\bf x}^{[n]})),
$$ 
where $H_s({\bf x})$ is the thresholding operator that preserves the largest $s$  entries (in magnitude), and $\mu^{[n]}>0$ is an adaptive step size parameter.
The recovery performance of Normalized {IHT} depends conditionally on the step size parameter $\mu^{[n]}$, unlike the traditional {IHT} approach in which $\mu^{[n]}=1$.
While the traditional approach requires a re-scaling of the measurement matrix such that $\|{\bf \Phi}\| <1$ to ensure convergence, introducing a step size parameter that enables the arbitrary scaling of ${\bf \Phi}$, and hence relaxes the bounds on its norm. Specifically, the role of $\mu^{[n]}$ is to compensate for this rescaling by avoiding the undesirable amplification of noise, i.e., by keeping the ratio $\|{\bf \Phi x}\|/\|{\bf e}\|$ unchanged. 

\paragraph*{Step size determination}
Normalized IHT adaptively sets the step size as follows: if the support of ${\bf x}^{[n]}$ is preserved between iterations, one can set the step size adaptively to 
\begin{equation}\label{step_size}
{\small
     \mu^{[n]} = \frac{\big({\bf g}^{[n]}_{\Gamma^{[n]}}\big)^T{\bf g}^{[n]}_{\Gamma^{[n]}}}{\big({\bf g}^{[n]}_{\Gamma^{[n]}}\big)^T{\bf \Phi}^{T}_{\Gamma^{[n]}}{\bf \Phi}_{\Gamma^{[n]}}{\bf g}^{[n]}_{\Gamma^{[n]}}},}
\end{equation}
where ${\bf g}^{[n]} = {\bf \Phi}^{T}({\bf y}-{\bf \Phi x}^{[n]})$ and $\Gamma^{[n]} = \supp({\bf x}^{[n]})$. This is shown to result in the maximal reduction of the cost function. However, if the support of ${\bf x}^{[n+1]}$ differs from that of ${\bf x}^{[n]}$, a sufficient convergence condition is shown to be
\begin{equation*}
    \mu^{[n]} \leq(1-c) \frac{\|{\bf x}^{[n+1]}-{\bf x}^{[n]} \|^2}{\|{\bf \Phi}({\bf x}^{[n+1]}-{\bf x}^{[n]}) \|^2}
\end{equation*}
for any small constant $c$. If the above condition is not met, a new proposal for ${\bf x}^{[n+1]}$ can be calculated by setting $\mu^{[n]}\leftarrow{\mu^{[n]}/(k(1-c))}$, where $k$ is a shrinkage parameter satisfying $k>1/(1-c)$.

A detailed description of Normalized IHT is given in Algorithm~\ref{Algorithm: Normalized_IHT}.

\subsection{Recovery Guarantees}

The analysis of hard thresholding algorithms  relies on the scaling properties of ${\bf \Phi}$. Concretely, one often considers the non-symmetric Restricted Isometry Property ({RIP}) condition: a matrix ${\bf \Phi}$ satisfies the non-symmetric {RIP} if there are $0<\alpha_s, \beta_s \in \mathbb{R}$ and $\alpha_s\leq \beta_s$ such that
\begin{equation}\label{eqn:rip}
\small
\alpha_{s} \leq \frac{\|{\bf \Phi} {\bf x}\|}{\|{\bf x}\|} \leq \beta_{s} { \ \textrm{for all} \ {\bf x} \ \textrm{with} \ \|{\bf x}\|_0 \leq s.}
\end{equation}
$\alpha_s$ and $\beta_s$ are the so-called Restricted Isometric Constants ({RIC}s). Note that for any support set $\Gamma$ such that $|\Gamma| \leq s$, $\alpha_s$ and $\beta_s$ are lower and upper bounded by the smallest and largest singular values of ${\bf \Phi}_{|\Gamma|}$, respectively.

The main convergence result of normalized {IHT} can be stated as follows~\cite{blumensath2012greedy}.

\begin{theorem} \label{theorem_convergence_IHT}
Let ${\bf \Phi}$ be full rank and $s\leq m$. If $\beta_{2s}\leq \mu^{-1}$, then normalized {IHT} converges to a local minimum of~(\ref{cost_function}).
\end{theorem}

When setting the step size parameter, the condition $\beta_{2s}\leq\mu^{-1}$, which ensures convergence, poses a challenge. To date, 
there is no efficient strategy to determine the exact values of the {RIC}s ${\beta_s}$ and $\alpha_s$ for an arbitrary measurement matrix in a computationally efficient manner. However, these constants can be bounded efficiently, and it can be shown that randomly constructed measurement matrices can satisfy the {RIP} with high probability~\cite{candes2008rip, chartrand2008rip}.

The adaptive setting of the step size parameter is further shown to provide a non-symmetric RIP variant recovery result as follows~\cite{blumensath2010niht}.

\begin{theorem}\label{guarantee_niht}
Consider a noisy observation ${\bf y} = {\bf \Phi x} + {\bf e}$ with an arbitrary vector ${\bf x}$, and let ${\bf x}^s$ be the best $s$-term approximation of ${\bf x}$. If $\rank({{\bf \Phi}})=M$ and $\rank({{\bf \Phi}_{\Gamma}})=s$ for all $\Gamma$ with $|\Gamma|=s$, then the normalized {IHT} algorithm converges to a local minimum of the cost function in~(\ref{cost_function}). Also, assume ${\bf \Phi}$ has the non-symmetric {RIP} when projected onto $2s$-sparse vectors, with RICs $\alpha_{2s}$ and $\beta_{2s}$. 

We further define $\gamma_{2s} =\beta_{2s}/\alpha_{2s}-1$ if the normalized IHT algorithm uses the step size defined in (\ref{step_size}) at each iteration, and $\gamma_{2s} =\max (1-\alpha_{2s}/k\beta_{2s},\ \beta_{2s}/\alpha_{2s}-1)$ otherwise, where $k>1$ is a shrinkage parameter introduced earlier. If $\gamma_{2s} \leq {1}/{8}$, then the recovery error after $n$ iterations is bounded as
\begin{equation}
    \| {\bf x}-{\bf x}^{[n]}\| \leq 2^{-n}\| {\bf x}^s\| + 8 {{\bf \epsilon}_s}\label{NIHT_error_bound},
\end{equation} where 
\begin{equation}\label{epsilon_s}
     {{\bf \epsilon}_s} =  \| {\bf x}-{\bf x}^s\| + \frac{\| {\bf x}-{\bf x}^s\|_1}{\sqrt{s}}+\frac{1}{\beta_{2s}}\|{\bf e} \|.
\end{equation}
\end{theorem}
\begin{corollary}
After at most $n^*=\log_2(\|{\bf x}^s\|/ {\bf \epsilon}_s)$ iterations, the recovery error bound in~(\ref{NIHT_error_bound}) can be further simplified to
\begin{equation*}
   \| {\bf x}-{\bf x}^{[n]}\| \leq 9 {{\bf \epsilon}_s}.
\end{equation*}
\end{corollary}
The above result suggests that, after a sufficiently large number of iterations, the reconstruction error is induced only by the noise ${\bf e}$ and that ${\bf x}$ is not exactly $s$-sparse.

\section{{QIHT}: Low Precision Iterative Thresholding}\label{Sec:QIHT}

We will now introduce the quantized version of normalized {IHT}, called {QIHT}, and analyze it in terms of signal recovery performance. 
The key idea here is that, by reducing the bit widths of the data points in a structured manner, we can upper bound the recovery error and fine tune the bit precision to still guarantee provable recovery performance. In Section~\ref{Sec:NumericalExperiments},
we will show that the recovery error bound reflects the true scaling of parameters in the regime where the non-symmetric RIP holds, and that for specific applications,
in particular radio astronomy and magnetic resonance imaging,
we expect the recovery error to be small, thanks to
the structure of the measurement matrix.

\subsection{The Algorithm}

Recall the quantized IHT iteration, assuming ${\bf x}^{[0]} = 0$: 
\begin{multline}\label{modified_update_rule}
  {\bf x}^{[n+1]} = \\
  H_s\big({\bf x}^{[n]}+\hat{\mu}^{[n]} Q(\boldsymbol{\Phi}^{T}, b_{\boldsymbol{\Phi}})(Q({\bf y}, b_{\boldsymbol{y}})-Q(\boldsymbol{\Phi}, b_{\boldsymbol{\Phi}}){\bf x}^{[n]})\big), 
\end{multline}
where the step size $\hat{\mu}^{[n]}$ is determined based on~(\ref{step_size}), and $Q(\cdot, b)$ is an element-wise quantization operator that maps single-precision floating-point values to $b$-bit precision. 

In the following, we will use the stochastic quantization operator
$Q(v, b)$, which quantizes ${v}$ to $b$-bit precision as follows. 
Let $\ell=2^b$ and $q_1, \dots, q_\ell$ denote $\ell$ equally spaced points in $[-1, 1]$ such that $q_1=-1\leq q_2 \leq \dots \leq q_\ell=1$. Assume that $v \in [q_i, q_{i+1}]$ for some $i$. Stochastic quantization maps $v$ to one of the two nearest points as follows:
$$    
Q(v, b) = \begin{cases}
        q_i, & \textrm{with probability} \ \frac{q_{i+1}-v}{q_{i+1}-q_i},\\
        q_{i+1},&\textrm{otherwise}.  \ \ \ \  \ \ \ \ \ \ \ \ \ \ \ \ \ \ 
        \end{cases} 
$$
Note that the quantization $Q(\cdot, b)$ is unbiased, i.e., $\mathbb{E}[Q({v}, b)] = {v}$, and matrices and vectors are quantized element-wise.

Note that in~(\ref{modified_update_rule}), we use two independent stochastic quantizations for ${\bf \Phi}^T$ and ${\bf \Phi}$, the so-called double sampling~\cite{zhang2017zipml}. This leads to an unbiased gradient estimator, i.e., 
\[
\mathbb{E}\big[ Q(\boldsymbol{\Phi}^{T}, b_{\boldsymbol{\Phi}})(Q({\bf y}, b_{\boldsymbol{y}})-Q(\boldsymbol{\Phi}, b_{\boldsymbol{\Phi}}){\bf x}^{[n]}\big)\big] = \boldsymbol{\Phi}^{T}({\bf y}-\boldsymbol{\Phi}{\bf x}^{[n]}),
\]
which provides better convergence results.

A detailed description of our quantized IHT is given in Algorithm~\ref{Algorithm:QIHT}. 

\begin{figure}
\begin{algorithm}[H]
{ 
\SetAlgoLined
{\bf Input:} {number of iterations $n^*$, $2n^*$ realizations of the low precision measurement matrix $Q({\bf \Phi})$: $\hat{\bf \Phi}_1, \hat{\bf \Phi}_2, \dots,\hat{\bf \Phi}_{2n^*}$, $n^*$ realizations of the low precision observation vector $Q({\bf y})$: $\hat{\bf y}_1, \hat{\bf y}_2, \dots,\hat{\bf y}_{n^*}$, sparsity parameter $s$, step size tuning parameters {\it k,\ c}\;}

{\bf Output:} The recovery vector ${\bf x}^{[n^*]}$

{Initialize} ${\bf x}^{[0]} = 0$, $\Gamma^{[0]} = \supp\big (H_s(\hat{\bf \Phi}_1^{T}{\hat{\bf y}})\big)$.

 \For{{$n=1$} {\bf to} {${n}^*$}}{
 ${\bf g}^{[n-1]} = {\hat{\bf \Phi}^{T}}_{2n-1} \big(\hat{{\bf y}}-{\hat{\bf \Phi}}_{2n}{\bf x}^{[n]}\big)$\;
 
 \vspace{0.25em}
 
 $\hat{\bf \mu}^{[n-1]} =...$\;
 
  {\scriptsize${\big({\bf g}^{[n-1]}_{\Gamma^{[n-1]}}\big)^T{\bf g}^{[n-1]}_{\Gamma^{[n-1]}}}/{\big( ({\bf \Phi}_{2n-1})_{\Gamma^{[n-1]}}{\bf g}^{[n-1]}_{\Gamma^{[n-1]}}}\big )^T({\bf \Phi}_{2n})_{\Gamma^{[n-1]}}{\bf g}^{[n-1]}_{\Gamma^{[n-1]}}$}\;
  
  \vspace{0.25em}
 
 ${\bf x}^{[n]} = H_s({\bf x}^{[n-1]} + \hat{\bf \mu}^{[n-1]} {\bf g}^{[n-1]})$\;
 
 \vspace{0.25em}
 
 $\Gamma^{[n]} = \supp({\bf x}^{[n]})$\;
 \vspace{0.25em}

\eIf{$\Gamma^{[n]}=\Gamma^{[n-1]}$}{

\vspace{0.15em}

${\bf x}^{[n]}={\bf x}^{[n-1]}$\;

}
{
${ b}^{[n]}=(\|{\bf x}^{[n]}-{\bf x}^{[n-1]} \|^2_2)/(\|\hat{\bf \Phi}_{2n-1}({\bf x}^{[n]}-{\bf x}^{[n-1]})\|^2_2)$\;

\vspace{0.25em}

\eIf{$\hat{\mu}^{[n]}\leq (1-{c})b^{[n]}$}{${\bf x}^{[n]}={\bf x}^{[n-1]}$}{

\vspace{0.25em}

\While{$\hat{\mu}^{[n]}> (1-{c})b^{[n]}$}{$\hat{\mu}^{[n]} = \hat{\mu}^{[n]}/(k(1-c))$\;

\vspace{0.25em}

${\bf x}^{[n]} = H_s({\bf x}^{[n-1]} + \hat{\bf \mu}^{[n-1]} {\bf g}^{[n-1]})$\;

\vspace{0.25em}
}
}
} 
 $\Gamma^{[n]} = \supp({\bf x}^{[n]})$\;
}
 \caption{{QIHT}: Low Precision {IHT}}
 \label{Algorithm:QIHT}
  }
\end{algorithm}
\end{figure}

\subsection{Main Results}

\paragraph*{Convergence}
We start with stating the convergence result. In the following, we set 
$\hat{{\bf \Phi}}= Q({\bf \Phi}, b_{\bf \Phi})$ and $\hat{{\bf y}}= Q({\bf y}, b_{\bf y})$.
We also use $\hat{{\bf \Phi}}_j$ to denote the $j^{th}$ quantization $\hat{{\bf \Phi}}$.

\begin{theorem}
The QIHT algorithm attains a local minimum of the cost function $\mathbb{E} [ \| \hat{\bf y} - \hat{\bf \Phi}{\bf x}\|^2]$ such that $\|{\bf x}\|_0 \leq s$. 
\end{theorem}
\begin{proof}
$\mathbb{E} [ \| \hat{\bf y} - \hat{\bf \Phi}{\bf x}\|^2]$ can be majorized by the following surrogate objective function 
\begin{equation*}
\small
    \begin{split}
        \mathbb{E} [\|\mu^{0.5}\hat{\bf y} -   \hat{\bf \Phi}{\bf x}\|^2+ \|{\bf x} 
    - {\bf x}^{[n]} \|^2
    - \|\mu^{0.5}\hat{\bf \Phi}({\bf x}- {\bf x}^{[n]}) \|^2], 
    \end{split}
\end{equation*}
whenever $\| \mu^{0.5}\hat{\bf \Phi}\|^2 <1$. This condition is met due to the step size determination introduced in~(\ref{step_size}). The minimizer of the above surrogate objective ${\bf x}^{[n+1]}$, therefore, ensures that $\mathbb{E} [\| \hat{\bf y} - \hat{\bf \Phi}{\bf x}^{[n+1]}\|^2] \leq \mathbb{E}[ \| \hat{\bf y} - \hat{\bf \Phi}{\bf x}^{[n]}\|^2]$. Using the arguments of~\cite{blumensath2008iht}, ~(\ref{modified_update_rule}) can be shown to minimize the expected cost $\mathbb{E} [ \| \hat{\bf y} - \hat{\bf \Phi}{\bf x}\|^2]$. 
\end{proof}

\subsubsection*{Performance Guarantees}

The following theorem states our main analytic result, which characterizes the recovery error of QIHT, specifically focusing on the additional error introduced by the quantization procedure.

\begin{theorem}\label{main_theorem_TH}
Consider an $M$-dimensional noisy observation vector ${\bf y} = {\bf \Phi x}+{\bf e}$ where ${\bf \Phi}$ is an $M\times N$-dimensional real or complex matrix, and ${\bf x}$ is an $N$-dimensional arbitrary vector. Let $H_s({\bf x}) = {\bf x}^s$ with $s\leq M$ and assume the full precision measurement matrix ${\bf \Phi}$ and the quantized measurement matrix $\hat{\bf \Phi}$ satisfy the non-symmetric RIP in~(\ref{eqn:rip}) and (\ref{Eq:rip_lowprecision}), with RICs ${\alpha}_s, {\beta}_s$ and $\hat{\alpha}_s, \hat{\beta}_s$, respectively. We also define $\gamma_{2s} =\beta_{2s}/\alpha_{2s}-1$ if the normalized IHT algorithm uses the step size defined in (\ref{step_size}) at each iteration and $\gamma_{2s} =\max (1-\alpha_{2s}/k\beta_{2s},\ \beta_{2s}/\alpha_{2s}-1)$ otherwise. Similarly, let $\hat{\gamma}_{2s} =\hat{\beta}_{2s}/\hat{\alpha}_{2s}-1$ if the QIHT algorithm uses the step size defined in (\ref{step_size}) and $\hat{\gamma}_{2s} =\max (1-\hat{\alpha}_{2s}/k\hat{\beta}_{2s},\ \hat{\beta}_{2s}/\hat{\alpha}_{2s}-1)$ otherwise. If ${\gamma}_{2s}$ and $\hat{\gamma}_{2s}$ satisfy ${\gamma}_{2s}\leq 1/24$ and $\hat{\gamma}_{2s}\leq 1/24$, then at each iteration $n$, the QIHT algorithm outputs an approximation of ${\bf x}^s$, ${\bf x}^{[n]}$ such that
\begin{equation}\label{main_theorem_paper}
        \mathbb{E}[\|\hat{{\bf x}}^{[n]}-{\bf x}^s\|]  
        \leq 2^{-n}\|{\bf x}^{s}\|  + 9 {\epsilon}_s + 4.5 {\epsilon}_q,
\end{equation}
where $\epsilon_s$ is given by
$$ 
{{\bf \epsilon}_s} =  \| {\bf x}-{\bf x}^s\| + \frac{\| {\bf x}-{\bf x}^s\|_1}{\sqrt{s}}+\frac{1}{\min(\beta_{2s}, \hat{\beta}_{2s})}\|{\bf e} \|
$$
and
\begin{equation*}
 {\epsilon}_q = \frac{\sqrt{M}}{\hat{\beta}_{2s}}\Big ( \frac{\| {\bf x}^s\|}{2^{b_{\bf \Phi}-1}}+\frac{1}{2^{b_{\bf y}-1}} \Big ).
\end{equation*}
Here, {b}${}_{\bf \Phi}$ and {b}${}_{\bf y}$ are the number of bits used to represent ${\bf \Phi}$ and ${\bf y}$, respectively.
\end{theorem}
\begin{proof}
See Appendix~\ref{Sec:proof of main theorem}.
\end{proof}

\begin{corollary}
A natural stopping criterion is $n^* = \lceil \log_2 (\|{\bf x}^s \|/\epsilon_s)\rceil$, which means the algorithm computes successive approximations of ${\bf x}^s$ with accuracy $\mathbb{E}[\|{\bf x}^{[n^*]}-{\bf x}^s \|] \leq 10\epsilon_s + 4.5\epsilon_q$.
\end{corollary}
\begin{proof}
Inserting $n^* = \lceil \log_2 (\|{\bf x}^s \|/\epsilon_s)\rceil$ into the $2^{-n}\|{\bf x}^{s}\|$ term in~(\ref{main_theorem_paper}) yields the result.
\end{proof}

\paragraph*{Determining the bit precision $b$} 
One constraint in the above theorem is that both $\hat{\bf \Phi}$ and $\bf \Phi$ satisfy the non-symmetric RIP with $\gamma_{2s}, \hat{\gamma}_{2s}\leq 1/24$.
The following lemma describes the relationship between the non-symmetric {RIP} properties of
${\bf \Phi}$ and  $\hat{\bf \Phi}$. The result suggests that one can ensure that the non-symmetric {RIP} holds for $\hat{\bf \Phi}$ using sufficient bit precision. 

\begin{lemma}
Let $\epsilon > 0$ and let $\bf \Phi_{\Gamma}$ satisfy the non-symmetric {RIP} with $\gamma_{|\Gamma|} \leq 1/24-\epsilon$ for any support set $\Gamma$. If ${b}_{\bf \phi} \geq \log \big (\frac{2\sqrt{|\Gamma|}}{\epsilon \alpha_{|\Gamma|}}\big )$, then $\hat{\bf \Phi}_{\Gamma}$ is guaranteed to satisfy $\hat{\gamma}_{|\Gamma|}\leq 1/24$.
\end{lemma} 
\begin{proof}
See Appendix~\ref{Sec:Proofs}.
\end{proof}
\subsection{Comparison of QIHT and Normalized IHT: Discussion and Limitations}

We now examine the error bound provided by Theorem~\ref{main_theorem_TH}. We note that this bound is slightly simplified, i.e., our proof in the Appendix~\ref{Sec: proof of lemma} is tighter. From there, we conclude that the RIP condition is scaled by a small factor which lies in the interval $(2, 3)$ when the measurement matrix is quantized.

The {QIHT} algorithm is guaranteed to asymptotically provide a sparse approximation of ${\bf x}$ up to multiples of ${\epsilon}_s$ and ${\epsilon}_q$ in the noise term ${\bf e}$ when ${\gamma}_{2s}, \hat{\gamma}_{2s}\leq 1/12$, and with rate $2^{-n}$ when ${\gamma}_{2s}, \hat{\gamma}_{2s}\leq 1/24$. We refer to (\ref{final_error_residual}) for the details of the former. ${\epsilon}_s$ is the approximation error when ${\bf x}$ is represented by a sparse vector ${\bf x}^s$, and ${\epsilon}_q$ is the noise introduced by the quantization operator. 
Here we highlight the properties, and discuss the limitations.

\paragraph*{Condition on $\hat{\boldsymbol{\gamma}}_{2s}$} Compared to Normalized IHT, the condition under which the performance guarantee holds is stricter in our approach, i.e., ${\gamma}_{2s}, \hat{\gamma}_{2s}\leq 1/24$, whereas the standard analysis requires ${\gamma}_{2s}\leq 1/8$ (Theorem~\ref{guarantee_niht}) for the same rate of convergence. Although it is hard to meet this constraint in practice, the small scaling factor between the convergence rates suggests that we can still expect good practical performance in the low precision setting, similarly to high precision Normalized IHT.

\paragraph*{Limitations on ${\boldsymbol{\beta}}_{2s}$ and $\hat{\boldsymbol{\beta}}_{2s}$} In QIHT, the measurement matrix ${\bf \Phi}$ is scale-invariant, and rigorous theoretical guarantees are achievable provided its scaling onto sparse vectors is confined in certain intervals, i.e., the RIP condition. 

The recovery error bound satisfying (\ref{main_theorem}) depends on the error terms in (8), ${\epsilon}_s$ and ${\epsilon}_q$, which are inversely proportional to ${\beta}_{2s}$ and $\hat{\beta}_{2s}$, respectively. For sufficiently large values, which would compensate for $\|{\bf e}\|$ and $\sqrt{M} \|{\bf x}^s \|$, the low precision approach appears competitive with the unmodified algorithm where the recovery error is bounded by $9\epsilon_s$ in (\ref{NIHT_error_bound}). Furthermore, the scale-invariance property of the measurement matrix ${\bf \Phi}$ permits us to scale up ${\beta}_{2s}$, and hence $\hat{\beta}_{2s}$, retaining a strong recovery guarantee, similar to that of the full precision algorithm. Scaling ${\bf \Phi}$ has no effect on the RIP condition.
 
\paragraph*{On the quantization error $\epsilon_q$} From the definition of $\epsilon_q$ we infer that the quantization errors introduced by the low precision measurement matrix and the measurements individually differ by a scale factor of $\|{\bf x}_s \|$ for the same bit widths. We argue that the approximation error caused by quantizing the measurement matrix would get smaller as $s$ gets smaller. Moreover, the scale invariant property of the measurement matrix can enable $\|{\bf x}_s\| < 1$ to hold, yet can potentially strengthen the effect of noise.

\paragraph*{Comparison to other state-of-the-art} The compressive sensing literature covers a range of algorithms including $\ell_1$-minimization and greedy- and thresholding-based methods, each with its own trade-offs. {CoSaMP}, normalized {IHT} and $\ell_1$-minimization exhibit similar empirical performance in~\cite{blumensath2010niht}, when applied to the problems with dense Gaussian matrices. Moreover, after tuning of the step size parameter, Normalized {IHT} is competitive to these powerful methods with similar provable guarantees~\cite{blumensath2012greedy}. Considering that the performance of Normalized IHT compared to other state-of-the-art methods is already well-studied in the literature and is superior in most cases, we focus only on comparing QIHT to the Normalized IHT in this paper. 

\section{Numerical Experiments}\label{Sec:NumericalExperiments}

The goal of this section is to examine the practical performance of our method. 
To provide more intuition, we first run synthetic experiments to quantify the performance gap between {QIHT} and Normalized {IHT} on a toy example: artificially generated data where the data points are drawn from independent and identically distributed (i.i.d.) Gaussian distributions. It is shown, for instance in~\cite{xu2014GaussianRIP, baraniuk2008RIP}, that the Gaussian matrices satisfy the RIP with high probability. The choice of such experimental data therefore helps us to better understand how the performance gap scales with the reduced number of precision levels, in a regime where the theoretical conditions do hold.

We then extend our focus to real-world larger-scale problems from radio astronomy and magnetic resonance imaging. We apply {QIHT} to (a) the radio interferometer measurements recorded by a real telescope: the LOw Frequency ARray (LOFAR), and (b) $k$-space subsamples recorded from the two-dimensional Fourier domain of a representative brain image. For both applications, we model the imaging problem in the compressive sensing framework and demonstrate that the accuracy achieved in the low precision setting is comparable with the one obtained by high precision solvers. Finally, we examine the speed-ups obtained by {FPGA} and {CPU} implementations.

\subsection{Experiments on Synthetic Data}

\paragraph*{Data} We draw the entries of $\mathbf{\Phi}\in \mathbb{R}^{128\times 1024}$, $\mathbf{x}^s\in \mathbb{R}^{1024}$ and $\mathbf{e}\in \mathbb{R}^{128}$ from an i.i.d. Gaussian distribution with zero mean and unit variance, $\mathcal{N}(0,1)$, such that the sparsity of $\mathbf{x}$, i.e., $s=|\supp(\mathbf{x})|$, varies from 4 to 128 in steps of 4. 

\paragraph*{Accuracy} We first compare the recovery performance of {QIHT} to {IHT} in the absence of noise, i.e., $\mathbf{y}=\mathbf{\Phi}\mathbf{x}$. To quantify the performance gap, we estimate (a) the recovery error: $\|\mathbf{x}^{n}- \mathbf{x}\|/\|\mathbf{x}\|$, and (b) the support recovery, i.e., the normalized support of $\mathbf{x}$ that is successfully recovered. We estimate the above measures by averaging over 100 realizations of data.
\begin{figure*}[t!]
 \centering
   \includegraphics[width=0.98\textwidth]{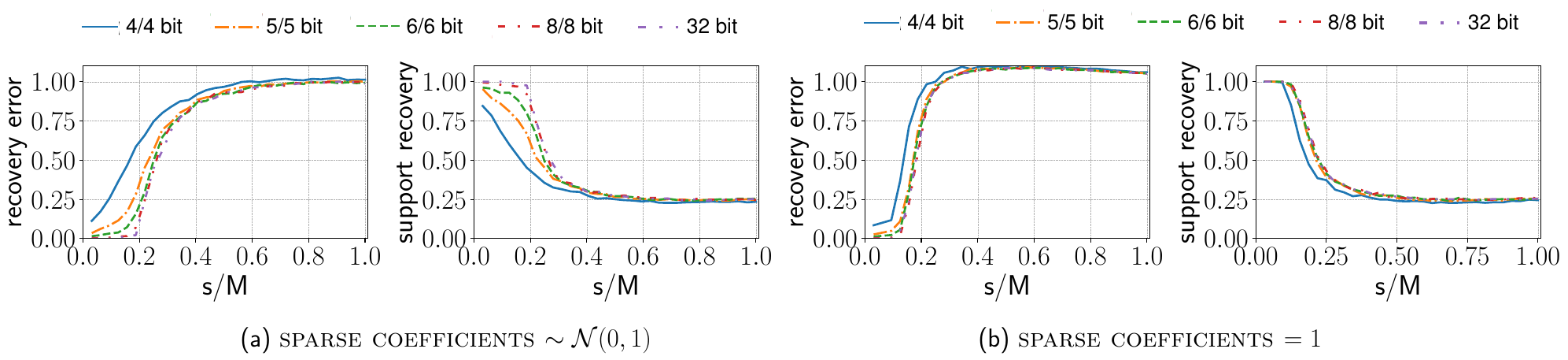}
\caption{Comparison of recovery error and support recovery for various precision levels.}
   \label{Fig:synthetic_noiseless}
\end{figure*}

\begin{figure*}[t!]
 \centering
   \includegraphics[width=1\textwidth]{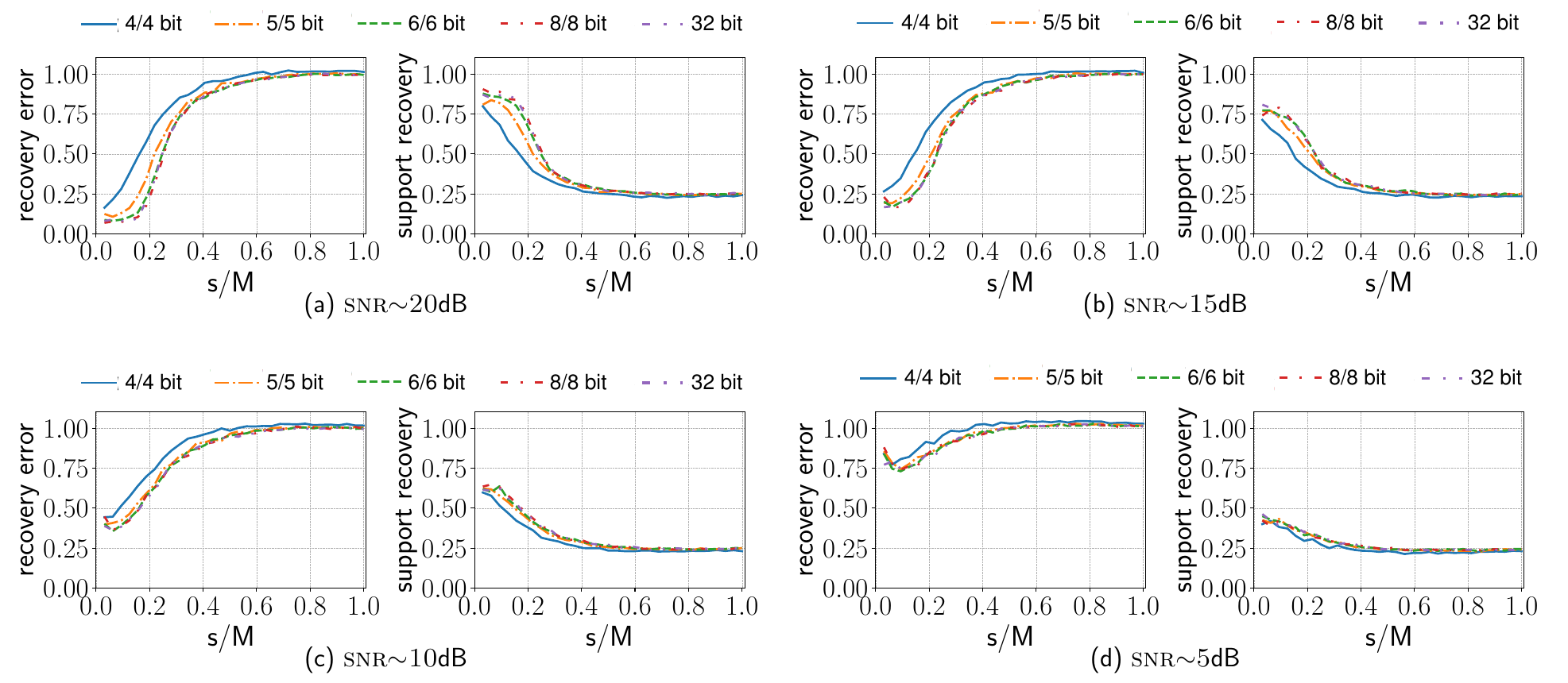}
\caption{Comparison of the recovery error and the support recovery at different precision levels and for different amounts of noise corruptions.}\label{Fig:synthetic_noise}
\end{figure*}

The results shown in Figure~\ref{Fig:synthetic_noiseless}(a) indicate that {QIHT} can achieve a recovery performance that is close to the normalized {IHT} even when as few as 5 bits are used. As expected, {QIHT} performs slightly worse when the precision is too aggresively lowered, for instance, down to $b_{{\bf \Phi}}=4$ and $b_{{\bf y}}=4$ bits. Yet the precision levels that preserve the quality of the results still can provide a significant speed-up in computation time for recovery.

We now consider a case, which is often considered challenging for sparse recovery algorithms: when the entries of $\mathbf{x}$ are of equal magnitude. We repeated the above experiment by setting the nonzero entries of $\mathbf{x}$ to 1 and demonstrate the performance of {QIHT} at several precision levels in Figure~\ref{Fig:synthetic_noiseless}(b). While Normalized {IHT} in general seems markedly worse at this setting, {QIHT} yields a performance as good as its high precision variant and the recovery gap between both methods becomes negligible.

This can be justified as follows: When the data is at low precision, the sparse coefficients are recovered by repeatedly using a linear transformation that has only a few number of precision levels, enforcing the sparse coefficients that are close to each other to be recovered as the similar magnitude. Hence, when the coefficients are of the same value, the precision of the recovered values is not important, leading eventually to less recovery error.

\paragraph*{Robustness to noise} 
In real-life applications, measurements are usually corrupted by noise. The theoretical bounds of the low precision variant of the Normalized {IHT} on recovery error, as given in~(\ref{NIHT_error_bound}), (\ref{epsilon_s}) and (\ref{main_theorem_paper}), suggest that lowering the precision of the input data only slightly increases the noise sensitivity. 

We therefore investigate the influence of lowering the precision on the recovery performance by corrupting the observations with different levels of noise. Figure~\ref{Fig:synthetic_noise} demonstrates the performance of {QIHT} for various levels of noise corruption, validating our theoretical observations that quantization does not amplify the effect of noise corruption on the sparse recovery.

Comparison of the Normalized {IHT} to other state-of-the-art methods such as {CoSaMP}, $\ell_1$-minimization for similarly generated artificial data is performed, for example, in~\cite{blumensath2010niht, blumensath2012greedy}. We defer to these references for further comparison.

\subsection{Real-World Applications}

Motivated by the success of {QIHT} on artificially generated data, we apply our framework to two larger scale real-world settings in radio astronomy and magnetic resonance imaging. The measurement matrix used in the compressive sensing formulation for these applications contains spatial information on a two-dimensional Fourier space, i.e., relative distances between entries are induced by the respective sensor locations. Therefore, useful information in the linear transformation matrix is preserved even when the precision is lowered, which results in little loss of visual information for the underlying image we aim to recover. In the following experiments, we show that this intuition is indeed correct and confirmed by the good performance of QIHT.

\paragraph{Radio Astronomy}
We consider a  radio astronomy application, in which radio interferometers at various locations on the ground record radio waves emitted by the celestial sources over a certain time interval, and then store and process these signals to deduce a sky image~\cite{hogbom1974clean}. Interferometers first estimate the cross-correlation between the time series measurements, called visibilities. The visibilities correspond to subsamples of a sky map in the Fourier domain where the sample point is a function of the antenna locations (van Cittert-Zernike theorem~\cite{taylor1999}). For the point source recovery problem, radio interferometry imaging inherently can be formulated as a sparse signal recovery problem. 

The usual strategy to date is to deconvolve the inverse Fourier transform of visibilities to form a sky map by iteratively removing a fraction of the highest peak, convolved with an instrument-based point spread function~\cite{hogbom1974clean}. Moreover, recently, the radio astronomy community has started to formalize
the radio interferometer problems also as
compressive sensing~\cite{wiaux2009csforra, wenger2010csforra, li2011deconvolution}. We directly follow this formulation.

The following is a standard formulation of the problem. Assume the sky is observed by employing $L$ antennas over a stationary time interval where the earth's rotation is negligible. Denote the vectorized sky image by ${\bf x} \in \mathbb{R}^N$ with $N = {r^2}$ where $r$ is the resolution of images, i.e., height and width of the image in pixels. 

We formulate the interferometer pipeline as a compressive sensing problem such that ${\bf y} = {\bf \Phi} {\bf x} + {\bf e}$ where ${\bf \Phi} \in \mathbb{C}^{M\times N}$ is the measurement matrix with complex entries as a function of the inner product between antenna and pixel locations, ${\bf y} \in \mathbb{C}^{M}$ contains the visibilities where $M=L^2$, and ${\bf e} \in \mathbb{C}^M$ is the noise vector. 

Real-life problems usually do not satisfy the traditional {RIP} condition $\|{\bf \Phi}\| < 1$. The scale-invariant feature of the measurement matrix used in Normalized {IHT} however alleviates the {RIP} issue and imposes a fairly mild constraint, i.e., the non-symmetric {RIP}. In a series of papers~\cite{blumensath2010niht, blumensath2012greedy}, {CoSaMP} is shown to perform markedly worse when the {RIP} condition fails. The Normalized {IHT}, however, still preserves its near-optimal recoveries far beyond the region of {RIP}. Motivated by this, we apply Normalized IHT and QIHT to real radio telescope data. 
 
Recall from Theorem~\ref{main_theorem_TH} the two conditions ensuring  performance guarantees: (a) $\gamma_{2s}, \hat{\gamma}_{2s} \leq 1/24$, and (b) ${\hat{\beta}_{2s}}$ must be large to minimize the quantization error $ {\epsilon}_q$. Fortunately, in the radio astronomy application, the image grid we initially form and the set of antennas used as well as the scale-invariant property of Normalized {IHT} give us control over $\gamma_{2s}, \hat{\gamma}_{2s}$ and ${\beta}_{2s}, \hat{\beta}_{2s}$, respectively, through a pre-processing of ${\bf \Phi}$. We leave this as future work.

We recover a sky image with a resolution of $256\times 256$ pixels (${\bf x} \in \mathbb{R}^{65,536}$ in vectorized form) by employing 30 low-band antennas of the LOFAR CS302 station that operate in the 15--80 MHz frequency band and in a Field of View (FoV) of 2 degrees where the sky is populated with 30 strong sources, that is, ${\bf y} \in \mathbb{C}^{900}$, $\hat{\bf \Phi} \in \mathbb{C}^{900\times 65,536}$. We note here that 30 antennas lead to a visibility matrix of size 30$\times$30, i.e., the measurement vector is of size 900. The signal-to-noise ratio ({SNR}) is assumed to be 5 dB at the antenna level, i.e., $10\log_{10}(\|{\bf \Phi x}\|^2/\|{\bf e}\|^2) = 5$ dB. 

Figure~\ref{summary_of_results}(a) provides an example of sky recoveries: (a) ground truth estimated over 12 hours of observation, (b) a least square estimate of underlying sky (or dirty image in the nomenclature of radio astronomy), (c) {32} bit and (d) {2}/{8} bit {QIHT} which uses {2} bit for the measurement matrix and {8} bit for the observation. This experiment indicates that {QIHT} captures the sky sources successfully even when only {2} bits are used to compress ${\bf \Phi}$. Thus, we can drastically reduce the data precision without significantly degrading the sky image quality.

This strong empirical performance is not completely surprising. Mathematically, the measurement matrix we formed here reflects the phase relations induced by the antenna locations.  That is, each time $r_{m, n}$ or $c_{m, n}$ flips its sign where $\hat{\bf \Phi}_{m, n} = r_{m, n} +j c_{m,n}$, $m = {1, 2, ..., M}$ and $n = {1, 2, ..., N}$, the change in horizontal and vertical directions on the ground enables preserving the phase information required for interferometric imaging even at very low precision.

\begin{figure*}[t!]
\centering
\includegraphics[width=0.8\textwidth, angle=0]{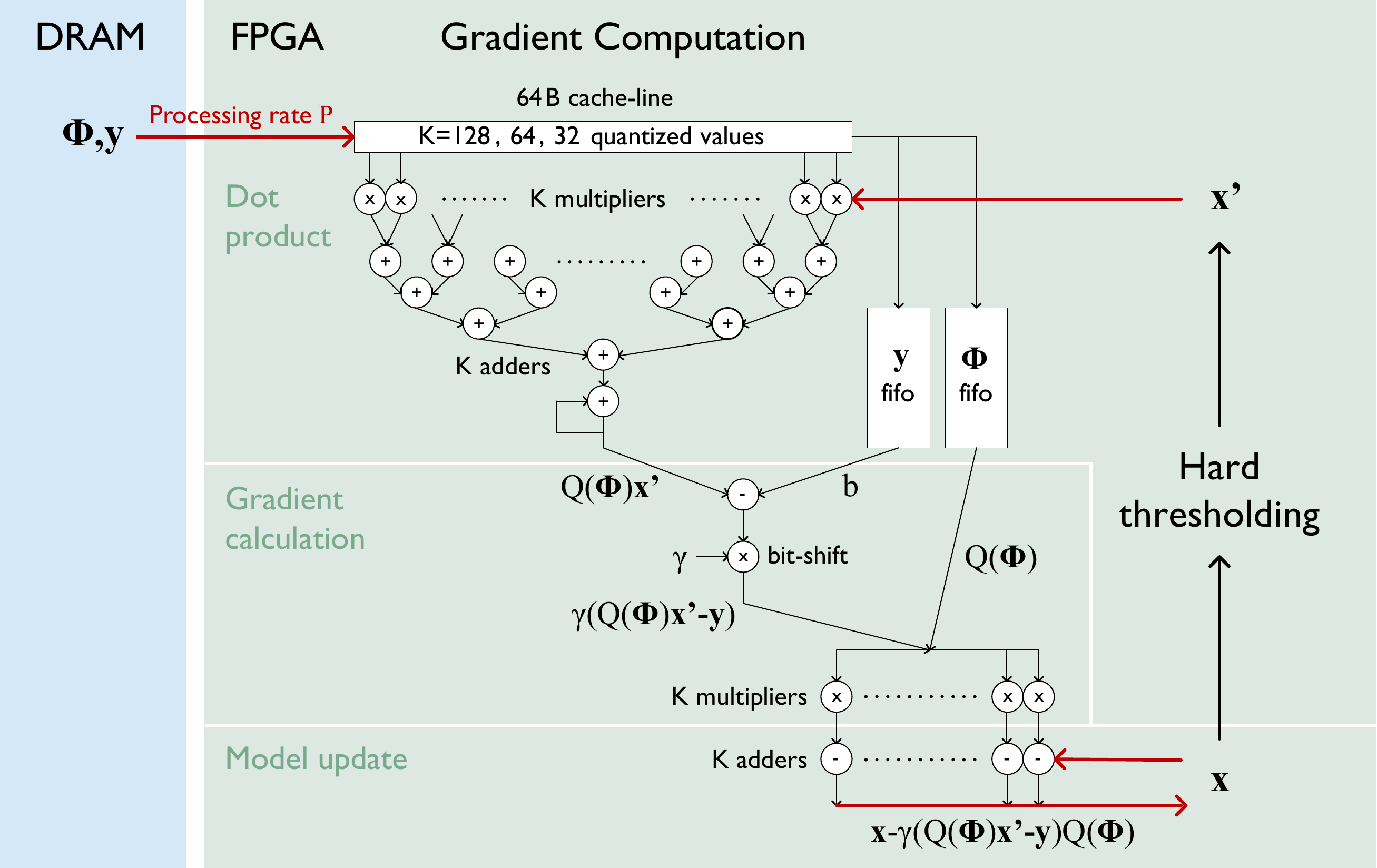}
\caption{{QIHT} on an {FPGA}-based system.}
\label{fig:fpga}
\end{figure*}

We evaluate {QIHT} through (1) the recovery error, and (2) the support recovery. In radio astronomy, it is customary to use a number of true celestial sources resolved in the recovered image as a performance metric, i.e., true-positive findings. That is, the performance of the algorithms is no longer described by its ability to recover support entirely but the sky objects, which possess higher error tolerance. 

\paragraph{Magnetic Resonance Imaging}

Compressive sensing enables faster magnetic resonance imaging ({MRI}) by acquiring less data through undersampling in the measurement space, hence accelerating the scan time. While \textit{Nyquist} criteria are violated due to the undersampling, the image is still reconstructed with little or no perceptible loss of visual information, established by a substantial body of work, for example~\cite{Lustig2007SparseMT,Lustig2008CompressedSM}.
The key ingredient behind this success is that magnetic resonance ({MR}) images exhibit a sparse representation in a known and fixed mathematical transform domain, i.e., the wavelet transform domain. A standard strategy is, therefore, to decode the sparse coefficients based on the undersampled measurements and store them for later encoding and reconstruction of the image. 

In {MRI}, the measurements are two-dimensional Fourier coefficients of the image, the so-called $k$-space samples. Inverse Fourier reconstruction of the image from the undersampled $k$-space data, however, is known to proudce aliasing artifacts. In order to mitigate undersampling artifacts, the compressive sensing algorithm iteratively finds an estimate of sparse coefficients. In our notation, $\boldsymbol{\Phi}$ is formed by Fourier and inverse wavelet transforms and sampling operator, $\textbf{x}$ has one-dimensional sparse coefficients, and finally, $\textbf{y}$ is a vector of undersampled $k$-space data. 

The performance of Normalized {IHT} on the {Shepp-Logan} phantom was previously studied in~\cite{blumensath2010niht}. Instead, we tested {QIHT} on a representative brain image${}^1$ of size 512$\times$512 in pixels and compare our results to the reconstructed image through $\ell_1$-minimization using the {SparseMRI} software\footnote{available on {\scriptsize {http://people.eecs.berkeley.edu/\string~mlustig/Software}}}. We subsample $k$-space data by a factor of 3 using a radial sampling mask.

The brain image reconstructed by various algorithms depicted in Figure~\ref{fig:result_highlight}(b) reveals that {QIHT} still yields a similarly good performance as the Normalized {IHT} and $\ell_1$-minimization when the bit-widths of the $k$-space data and the transformation matrices are lowered down to 8 and 12 bit, respectively. 

While offering accelerated image recovery for {MRI}, low precision data representation can potentially reduce the storage required to keep patients raw data as discussed in~\cite{poldrack_mumford_nichols_2011, Langer2011}.

\subsection{Implementation and Performance}

We demonstrate the speed-up obtained by performing {QIHT} in the previous two applications on both {FPGA} and {CPU} when reducing the number of bits used for the data representation.

\paragraph{{FPGA} implementation} 
Field-Programmable Gate Arrays ({FPGA}) are an alternative to commonly used Graphics Processing Units ({GPU}) for accelerated processing of compute-intensive signal processing workloads. The reconfigurable logic fabric of an {FPGA} enables the design of custom compute units, that can be advantageous when working on low-precision and uncommon numeric formats, such as {2}-bit numbers. Thanks to this microarchitectural flexibility, it is possible to achieve near linear speed-up when lowering the precision of data that is read from memory. This has been shown recently for stochastic gradient descent ({SGD}) when training linear models~\cite{zhang2017zipml, kara2017fpga}. In this work, we use the open-source {FPGA} implementation\footnote{\url{https://github.com/fpgasystems/ZipML-PYNQ}} from the above mentioned works and modify it to perform {QIHT}.

In terms of the computation, we modify two parts of the design to convert it from performing {SGD} to {IHT}. First, instead of updating the model after a mini-batch count is reached, we update it after all samples are processed and the true gradient is available. Second, after each epoch, we perform a binary search on the updated model to find the threshold value satisfying that only top $s$ values are larger than the threshold.
The rest of the design stays the same, including the fixed-point computation, utilized to minimize the usage of available FPGA resources.

\begin{figure*}[t!]
\subfloat[][]{
   \includegraphics[width=1\textwidth]{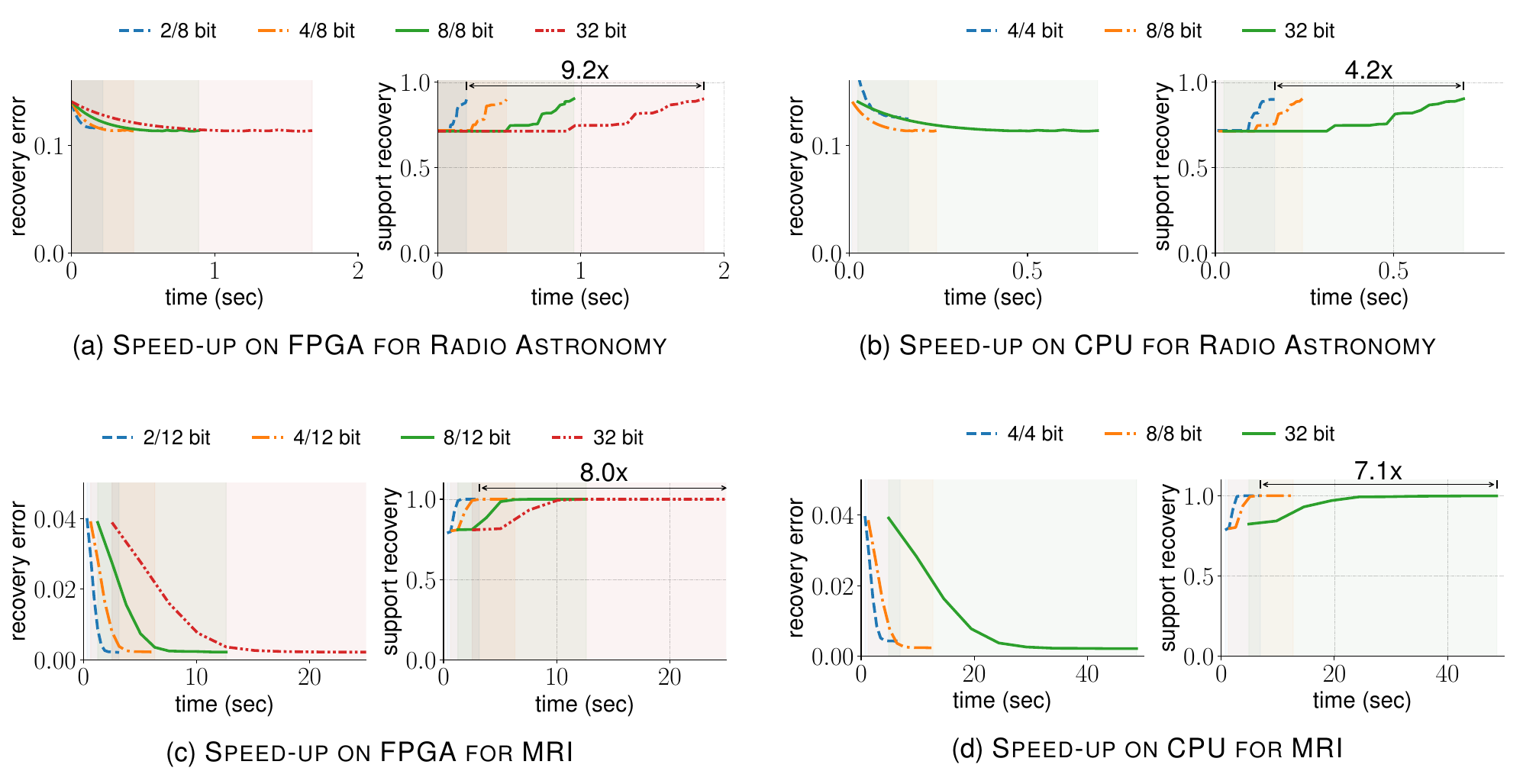}
 }
\caption{Speed-up on image recovery enabled by {QIHT} on FPGA and {CPU} on real radio telescope and MRI data.}\label{systems_speedup}
\end{figure*}

\paragraph*{FPGA performance analysis} 
The gradient computation unit in Figure~\ref{fig:fpga} reads the measurement matrix ${\bf \Phi}$ and the measurements ${\bf y}$ from the main memory and keeps ${\bf x}$ in on-chip memory. We note that transferring ${\bf \Phi}$ from main memory will be necessary in most practical settings, where the matrix ${\bf \Phi}$ is too large to fit onto the {FPGA}.
The {FPGA} is able to consume and process the data from the memory at a rate of $P=12.8$ GB/s. Thus, the performance is bounded by ${P}$ for processing ${\bf \Phi}$ and ${\bf y}$.
The time for each iteration is approximately 
$T = \size({\bf \Phi})/{P}$, since $size({\bf y}) \ll size({\bf \Phi})$. 
Theoretically, we can achieve a significant speed-up by using a quantized ${\bf \Phi}$, simply because we reduce the amount of data to be consumed by the FPGA: more entries arrive with each transfer from the main memory. 
The essential idea behind achieving linear speed-up is lowering the precision of ${\bf \Phi}$ while keeping ${P}$ constant. This is possible, because we can adapt the gradient computation unit's microarchitecture and increase its internal parallelism to handle more values per incoming line, thanks to the FPGA's architectural flexibility.

\paragraph*{Computing ${\bf \Phi}$ on the fly} The above analysis focuses on the case where ${\bf \Phi}$ is stored in main memory, in which case quantization helps to reduce the amount of data transferred between the main memory and {FPGA}. In some applications, ${\bf \Phi}$ can be calculated on the fly, inside the {FPGA}. Also in this case, quantization can help in achieving better performance.
The reason is that quantizing ${\bf \Phi}$ also saves other crucial resources (e.g., multipliers) that are limited on an {FPGA}. These resource savings, in turn, enable higher internal parallelism, for instance, to speed up the computation of ${\bf \Phi}{\bf x}$. For example, it has been shown that to increase the dot-product parallelism from 64 to 128 while maintaining the rate of operations per cycle, it is necessary to lower the precision of one side of the dot product to {2}-bits; otherwise, the resource consumption is too high to fit the design to one FPGA~\cite{kara2017fpga}.

The performance of the {FPGA}-based implementation is presented in Figure~\ref{systems_speedup}(a) and (c). For the time spent per iteration, we see that quantization, and the resulting compression of the measurement matrix ${\bf \Phi}$ leads to a near linear speed-up for recovering the support vector. All variants (full precision to lowest precision) of the Normalized {IHT} on {FPGA} can consume ${\bf \Phi}$ at the same rate, and therefore the runtime of {QIHT} depends linearly on the size of ${\bf \Phi}$, yielding the linear speed-ups that we observe in the experiments.
In terms of end-to-end performance, we measure the time needed for each precision level to reach support recovery ratio 90\% and calculate the speed-up. The {2}/{8} bit {QIHT} reaches the same support recovery ratio 9.2x faster.

\paragraph{{CPU} implementation} 

On a CPU, it is possible to achieve near-linear speedup when reducing the size of the data representation despite lacking the necessary instructions to compute with 4, 8, or 16-bit integer operands. This has been previously demonstrated for both gradient descent ({GD}) and {QIHT}~\cite{stojanov2018fast}.

In order to perform the low-precision computations on a CPU without instructions supporting low-precision arithmetic, low-precision data is first converted to 32-bit floating point. Thus the instructions used to for low-precision arithmetic are actually less efficient than using single-precision 32-bit arithmetic. The advantage on a {CPU} is that the low-precision representation results in less data movement.

To perform {CPU} experiments, we build on the implementation Clover from~\cite{stojanov2018fast}. The main extension needed was support for complex arithmetic. The bulk of the computation for both GD and IHT is dominated by two matrix-vector multiplication operations. The first is a dense matrix times a sparse vector, and the second a dense matrix times a dense vector. The former is implemented as a loop around a dot-product operation for {4}- and {8}-bit and uses the BLAS {\texttt gemv} routine for {32} bit. The latter is implemented as a loop around a dense scale and add operation for all three supported datatypes.

Our {CPU} implementation uses handwritten code in AVX2 intrinsics and supports {4}-bit, {8}-bit, and {32}-bit precisions. We use OpenMP to parallelize our implementation, XORShift to generate random numbers for stochastic rounding, and the Intel math kernel library (MKL) for the 32-bit matrix-vector multiplication. We used two different systems for our experiments. The radio astronomy experiment was run on an Intel Xeon CPU E3-1285L v3 3.10GHz, with 32GB of RAM and 25.6 GB/s bandwidth to main memory, running Debian GNU/Linux 8 (jessie), kernel 3.16.43-2+deb8u3. The MRI experiment was run on an Intel Xeon E5-2690 v4 CPU with 512 GB of RAM, 153.6 GB/s bandwidth to main memory, running Ubuntu 16.04.6 LTS with kernel version 4.4.0-148-generic. We use the Intel icc compiler 17.0.0, Intel IPP 2017.0.0 (r52494), and Intel MKL 2017.0.0 (Build 20160801). The {RDTSC} instruction is used to measure the cycle count for each iteration, and we report the median. Turbo Boost and Hyper-threading were disabled to avoid the effects of frequency scaling and resource sharing on the measurements.

We show performance plots for {CPU} speed-up in Figure~\ref{systems_speedup} (b) and (d). On both data sets, we obtain up to a 2.84x speed-up for the the 8-bit implementation, and 7.1x for the {4}-bit implementation, with similar recovery properties as for {FPGA}.

\section{Conclusion}

We investigated low precision schemes for sparse signal recovery problems with particular focus on the case in which both the observation vector and the measurement matrix are quantized. As main contribution, we introduced for this situation a low-precision Normalized {IHT} variant for stochastically quantized data, called {QIHT}. We derived theoretical guarantees and demonstrated good practical performance, both in terms of accuracy and recovery time in two application areas, radio astronomy and MRI, and with both {CPU} and {FPGA} implementations. Possible future work includes algorithms that work with end-to-end low precision data representation, and extensions to other greedy recovery algorithms and sparse recovery frameworks.

\section*{Acknowledgment}
The authors would like to thank Dr. Michiel Brentjens at the Netherlands Institute for Radio Astronomy (ASTRON) for providing radio interferometer data and Dr. Josip Marjanovic and Dr. Franciszek Hennel at the Magnetic Resonance Technology of ETH Zurich for providing their insights on the experiments.

\section*{Appendix}\label{Sec:Proofs}
\subsection{Preliminaries}
We begin by introducing our notation.
\begin{enumerate}
    \item ${\bf y} = {\bf \Phi} {\bf x} + {\bf e} = {\bf \Phi} {\bf x}^s + {\bf \Phi}({\bf x}-{\bf x}^s) + {\bf e}$
    \item $\boldsymbol{\varepsilon} = {\bf \Phi}({\bf x}-{\bf x}^s) + {\bf e}$, hence ${\bf y} = {\bf \Phi} {\bf x}^s + \boldsymbol{\varepsilon}$
    \item $\boldsymbol{\epsilon_y} = Q({\bf y},b) - {\bf y}$
    \item $\Gamma^{[n]} = \rm{supp}\{{\bf x}^{[n]} \}$, $\hat{\Gamma}^{[n]} = \rm{supp}\{\hat{{\bf x}}^{[n]} \}$ and\\
    $\Gamma^{s} = \rm{supp}\{{\bf x}^{s} \}$
    \item $\hat{B}^{[n]} = \hat{\Gamma}^{[n]} \cup \Gamma^{s}$
    \item ${\bf a}^{[n+1]} = \hat{{\bf x}}^{[n]} + \mu^{[n]}{\bf \Phi}^{\dag}({\bf y} - {\bf \Phi}\hat{{\bf x}}^{[n]})$ and\\
    $\hat{{\bf a}}^{[n+1]} = \hat{{\bf x}}^{[n]} + \hat{\mu}^{[n]}Q_{1}({\bf \Phi})^{\dag}({\bf y} - Q_2({\bf \Phi})\hat{{\bf x}}^{[n]})$,
    \item ${\bf x}^{[n+1]} = H_s({\bf a}^{[n+1]})$ and $\hat{{\bf x}}^{[n+1]} = H_s(\hat{{\bf a}}^{[n+1]})$
    \item ${\bf r}^{[n]} = \hat{{\bf x}}^{[n]}-{\bf x}^s$.
\end{enumerate}
Assume $\mathbf{\bf \Phi}$ satisfies the non-symmetric Restricted Isometry Property (RIP)
\begin{equation}\label{Eq:rip}
\alpha_{s} \leq \frac{\|{\bf {\bf \Phi}} {\bf x}\|_2}{\|{\bf x}\|_2} \leq \beta_{s}
\end{equation}
for all ${\bf x}: \|{\bf x}\|_0 \leq s$, where $\alpha_s\in \mathbb{R}$ and $\beta_s \in \mathbb{R}$ are the lowest and largest singular value of ${\bf \Phi}$ such that $0<\alpha_s\leq \beta_s$, the so-called restricted isometric constants. Inherent from its definition, the RIP for the quantized measurement matrix denoted by $Q({\bf {\bf \Phi}}, b_m)$ refers to that $\forall Q({\bf {\bf \Phi}}, b_m)$, we have
\begin{equation}\label{Eq:rip_lowprecision}
\hat{\alpha}_{s} \leq \frac{\|Q({\bf {\bf \Phi}}, b_m) {\bf x}\|_2}{\|{\bf x}\|_2} \leq \hat{\beta}_{s}
\end{equation}
where $\hat{\alpha}_{s}$ and $\hat{\beta}_{s}$ are the associated restricted isometry constants. For simplicity, we drop $b_m$, and use $Q({\bf {\bf \Phi}})$ instead. 

The adaptive setting of step size parameter $\mu^{[n]}$ in normalized {IHT} is shown to satisfy $1/\beta^2_{2s} \leq \mu^{[n]} \leq 1/\alpha^2_{2s}$ if $\mu^{[n]}$ is set to ${{\bf g}^{\dagger}_{\Gamma^{[n]}}{\bf g}_{\Gamma^{[n]}}}/{{\bf g}^{\dagger}_{\Gamma^{[n]}}{\bf \Phi}^{\dagger}_{\Gamma^{[n]}}{\bf \Phi}_{\Gamma^{[n]}}{\bf g}_{\Gamma^{[n]}}}$ at each iteration and $1/k\beta^2_{2s} \leq \mu^{[n]} \leq 1/\alpha^2_{2s}$ otherwise~\cite{blumensath2010niht}. Similar inequality also holds in the quantized setting such that $1/\hat{\beta}^2_{2s} \leq \hat{\mu}^{[n]} \leq 1/\hat{\alpha}^2_{2s}$ if $\hat{\mu}^{[n]}$ is set to ${\hat{\bf g}^{\dagger}_{\hat{\Gamma}^{[n]}}\hat{\bf g}_{\hat{\Gamma}^{[n]}}}/{\hat{\bf g}^{\dagger}_{\hat{\Gamma}^{[n]}}\hat{\bf \Phi}^{\dagger}_{\hat{\Gamma}^{[n]}}\hat{\bf \Phi}_{\hat{\Gamma}^{[n]}}\hat{\bf g}_{\hat{\Gamma}^{[n]}}}$ at each iteration, and $1/{k}\hat{\beta}^2_{2s} \leq \hat{\mu}^{[n]} \leq 1/\hat{\alpha}^2_{2s}$ otherwise.

Recall from Theorem~\ref{main_theorem_TH} that depending on the step size, $\gamma_s$ is defined as either $\beta_s/\alpha_s-1$ or $\max \{1-\alpha_s/k\beta_s, \ \beta_s/\alpha_s-1\}$. That also holds for $\hat{\gamma}_s$ by replacing the restricted isometry constants with that of quantized measurement matrix $Q(\boldsymbol{\Phi})$; $\hat{\alpha}_s$ and $\hat{\beta}_s$.\\ 
\begin{remark}\label{Remark:mu_bounds}
 Using the definitions of $\gamma_s$ and $\hat{\gamma}_s$ as well as bounds on $\mu^{[n]}$ and $\hat{\mu}^{[n]}$, we further have
\begin{equation*}\label{bounds_mu_gamma}
   (1-\gamma_{2s})/\alpha^2_{2s} \leq \mu^{[n]} \leq (1+\gamma_{2s})/\beta^2_{2s},
\end{equation*}
\begin{equation*}\label{bounds_muhat_gamma}
   (1-\hat{\gamma}_{2s})/\hat{\alpha}^2_{2s} \leq \hat{\mu}^{[n]} \leq (1+\hat{\gamma}_{2s})/\hat{\beta}^2_{2s}.
\end{equation*}
\end{remark}

Based on the properties above, the RIP and the adaptive step size, which we will require repeatedly throughout the proof, has several other consequences, summarized as follows.\\
\begin{lemma}\label{lemma_auxiliary_results}
Suppose ${\bf \Phi}$ and $Q(\boldsymbol{{\bf \Phi}})$ satisfy RIP in~(\ref{Eq:rip}) and (\ref{Eq:rip_lowprecision}), respectively. Let moreover $\Gamma, \Upsilon$ and $\Lambda$ has cardinality at most $\min\big(rank(\boldsymbol{\Phi}), rank(Q(\boldsymbol{\Phi}))\big)$ and $\Upsilon$ and $\Gamma$ are disjoint, $\Upsilon \cap \Lambda=\emptyset$. Then
\begin{equation*}\label{mu_phi}
\begin{split}
    \| \big({\mu}^{[n]}{\bf {\bf \Phi}}_{\Gamma}^T-\hat{\mu}^{[n]}Q({\bf {\bf \Phi}})_{\Gamma}^T\big){\bf x}_{\Gamma}\|_2\stackrel{(1)}\leq \max\big ( ({1+\gamma_{|\Gamma|}})/{\beta_{|\Gamma|}},{1+\hat{\gamma}_{|\Gamma|}})/{\hat{\beta}_{|\Gamma|}}\big ) \| {\bf x}_{\Gamma}\|_2, 
    \end{split}
\end{equation*} 
\begin{equation*}\label{mu_phi2}
\begin{split}
    \|\big( \mu^{[n]}{\bf \Phi}_{\Gamma}^T\boldsymbol{\Phi}_{\Gamma} - \mu^{[n]}Q_1(\boldsymbol{\Phi})_{\Gamma}^T Q_2(\boldsymbol{\Phi})_{\Gamma}\big){\bf x}_{\Gamma}\|_2\stackrel{(2)}\leq (\gamma_{|\Gamma|}+\hat{\gamma}_{|\Gamma|})\| {\bf x}_{\Gamma}\|_2
    \end{split}
\end{equation*}
\begin{equation*}\label{mu_phi2_diff}
\begin{split}
    \| \big(\mu^{[n]}\boldsymbol{\Phi}_{\Upsilon}^T\boldsymbol{\Phi}_{\Lambda} - \mu^{[n]}Q_1(\boldsymbol{\Phi})_{\Upsilon}^TQ_2(\boldsymbol{\Phi})_{\Lambda}\big){\bf x}_{\Lambda}\|_2\stackrel{(3)}\leq (\gamma_{|\Upsilon \cup \Lambda|}+ \hat{\gamma}_{|\Upsilon \cup \Lambda|})\| {\bf x}_{\Lambda}\|_2. 
    \end{split}
\end{equation*}
\end{lemma}
\begin{proof}
As a simple consequence of RIP, the singular values of ${\bf \Phi}_{\Gamma}$ lie between $\alpha_{|\Gamma|}$ and $\beta_{|\Gamma|}$. Remark~\ref{Remark:mu_bounds} further implies that the singular values of ${\mu}^{[n]}{\bf {\bf \Phi}}_{\Gamma}$ are in $[(1-\gamma_{|\Gamma|})/\alpha_{|\Gamma|}, (1+\gamma_{|\Gamma|})/\beta_{|\Gamma|}]$. Using the similar bound for $Q({\bf \Phi})_{\Gamma}$, maximum singular value of $({\mu}^{[n]}{\bf {\bf \Phi}}_{\Gamma}^T-\hat{\mu}^{[n]}Q({\bf {\bf \Phi}})_{\Gamma}^T)$, i.e., its operator norm, is given by $(1+\gamma_{|\Gamma|})/\beta_{|\Gamma|}-(1-\hat{\gamma}_{|\Gamma|})/\alpha_{|\Gamma|}$. In the first inequality of Lemma~\ref{lemma_auxiliary_results}, we use a looser bound $(1+\gamma_{|\Gamma|})/\beta_{|\Gamma|}$ for simplicity.

Similar argument holds for the second inequality, that is, the singular values of $\mu^{[n]}{\bf \Phi}_{\Gamma}^T\boldsymbol{\Phi}_{\Gamma}$ and $\mu^{[n]}Q_1(\boldsymbol{\Phi})_{\Gamma}^TQ_2(\boldsymbol{\Phi})_{\Gamma}$ fall into $[1-\gamma_{|\Gamma|}, 1+ \gamma_{|\Gamma|}]$ and $[1-\hat{\gamma}_{|\Gamma|}, 1+ \hat{\gamma}_{|\Gamma|}]$, respectively. Then $\|\mu^{[n]}\boldsymbol{\Phi}_{\Gamma}^T\boldsymbol{{\bf \Phi}}_{\Gamma} - \mu^{[n]}Q_1(\boldsymbol{\Phi})_{\Gamma}^TQ_2(\boldsymbol{\Phi})_{\Gamma}\|_2$ is upper bounded by ${\gamma}_{|\Gamma|} + \hat{\gamma}_{|\Gamma|}$, which proves the second inequality. 

The third inequality is a consequence of the fact that $-\mu^{[n]}\boldsymbol{{\bf \Phi}}_{\Upsilon}^T\boldsymbol{{\bf \Phi}}_{\Lambda}$ is a submatrix of $\rm{I}-\mu^{[n]}\boldsymbol{{\bf \Phi}}_{\Upsilon \cup \Lambda}^T\boldsymbol{{\bf \Phi}}_{\Upsilon \cup \Lambda}$
As previously shown, eigenvalues of $\mu^{[n]}\boldsymbol{{\bf \Phi}}_{\Upsilon \cup \Lambda}^T\boldsymbol{{\bf \Phi}}_{\Upsilon \cup \Lambda}$ lie in $[1-\gamma_{|\Upsilon \cup \Lambda|}, \ 1+\gamma_{|\Upsilon \cup \Lambda|} ]$. Hence, eigenvalues of $\mu^{[n]}\boldsymbol{{\bf \Phi}}_{\Upsilon}^T\boldsymbol{{\bf \Phi}}_{\Lambda}$ are in $[-\gamma_{|\Upsilon \cup \Lambda|}, \gamma_{|\Upsilon \cup \Lambda|}]$. The maximum eigenvalue of $(\mu^{[n]}\boldsymbol{\Phi}_{\Upsilon}^T\boldsymbol{\Phi}_{\Lambda} - \mu^{[n]}Q_1(\boldsymbol{\Phi})_{\Upsilon}^TQ_2(\boldsymbol{\Phi})_{\Lambda})$, hence its operator norm, can then be upper bounded by $\gamma_{|\Upsilon \cup \Lambda|} + \ \hat{\gamma}_{|\Upsilon \cup \Lambda|}$.
\end{proof}

\begin{lemma} \label{residual_lemma}
{\rm{\cite{blumensath2010niht}}}
For any ${\bf x}$, let ${\bf x}^s$ be the best s-term approximation to ${\bf x}$ and $\Upsilon$ be a set with at most s elements. Then
\begin{equation}\label{residual_x_bound}
\|\mu^{[n]}{\bf \Phi}^T_{\Upsilon} {\bf \Phi} ({\bf x}-{\bf x}^s)\|_2 \leq (1+\gamma_{2s})\Big [\|{\bf x}-{\bf x}^s\|_2 ]+ \frac{\|{\bf x}-{\bf x}^s\|_1}{\sqrt{s}}\Big].
\end{equation}
\end{lemma}

\begin{lemma}\label{lemma_on_quantized_vector}
Let $Q(\cdot, b): \mathbb{R}^d\times \mathbb{Z}^+ \rightarrow \mathbb{R}^d$ denote quantization operator described in the paper. For any ${\bf v}\in \mathbb{R}^d$, the norm of quantization error can be bounded by
\begin{equation}
    \mathbb{E}[\| Q({\bf v}, b) - {\bf v}\|_2] \leq \frac{c_v\sqrt{M}}{2^{b-1}}
\end{equation}
where $c_{\bf v}$ is the maximum value of the components of ${\bf v}$ in magnitude.
\end{lemma}
\begin{remark}
For efficient fixed-point computation on Field Programmable Gate Array, we need an odd number of quantization levels, and therefore total number of levels for $b$ bit quantization is $2^{b-1}+1$. That is, the interval between two consecutive levels is $1/2^{b-2}$ provided the values are confined in the interval $[-1, 1]$ a priori.
\end{remark}
\begin{proof}
Let $\tilde{\bf v} = {\bf v}/c_v$. Using Jensen's inequality we can easily show that 
\begin{equation*}\label{quantization_error_bound1}
\scriptsize
\begin{split}
    \mathbb{E}[\| Q(\tilde{\bf v}, b) - \tilde{\bf v}\|_2] \leq \sqrt{\mathbb{E}[\| Q(\tilde{\bf v}, b) - \tilde{\bf v}\|^2_2]}= \sqrt{\sum_{i=1}^{M} \mathbb{E}[\big (Q(\tilde{\bf v}, b)_i -\tilde{v_i}\big)^2]}
    \leq \sqrt{\sum_{i=1}^{M} \mathbb{P}(Q(\tilde{\bf v}, b)_i = \ell_j)(\tilde{v}_i-\ell_j)^2 + \mathbb{P}(Q(\tilde{\bf v}, b)_i = \ell_{j+1})(\ell_{j+1}-\tilde{v}_i)^2}.
\end{split}
\end{equation*}
Our quantization scheme uses a stochastic approach such that $\mathbb{P}(Q(\hat{\bf v}, b)_i = \ell_j)=\frac{\ell_{j+1}-\tilde{v}_i}{\ell_{j+1}-\ell_{j}}$, and hence $\mathbb{P}(Q(\tilde{\bf v}, b)_i = \ell_{j+1})=1-\frac{\ell_{j+1}-\tilde{v}_i}{\ell_{j+1}-\ell_{j}}$. Substituting these into the above inequality we have
\begin{equation}\label{quantization_error_sub}
\begin{split}
    \mathbb{E}[\| Q(\tilde{\bf v}, b) - \tilde{\bf v}\|_2] \leq \sqrt{\sum_{i=1}^{n} (l_{j+1}-Q(\tilde{\bf v}, b)_i)(Q(\tilde{\bf v}, b)_i-\ell_j)}.
\end{split}
\end{equation}

It can easily be seen that $(l_{j+1}-Q(\hat{\bf v}, b)_i)(Q(\hat{\bf v}, b)_i-\ell_j)$ is maximized when $Q(\hat{\bf v}, b)_i) =  \frac{\ell_{j+1}-\ell_j}{2}$, moreover the quantization function implies that $\ell_{j+1}-\ell_j = \frac{1-(-1)}{l} = \frac{1}{2^{b-2}}$

\begin{equation}\label{quantization_error}
\begin{split}
    \mathbb{E}[\| Q(\tilde{\bf v}, b) - \tilde{\bf v}\|_2] \leq \sqrt{\sum_{i=1}^{M} \frac{(\ell_{j+1}-\ell_j)^2}{4}}
    \leq \frac{\sqrt{M}(\ell_{j+1}-\ell_j)}{2}\leq \frac{\sqrt{M}}{2^{b-1}}.
\end{split}
\end{equation}
\end{proof}
\subsection{Proof of Theorem 3}\label{Sec:proof of main theorem}
The recovery error can be split into two parts by using triangle inequality
\begin{equation}\label{main_theorem}
\begin{split}
   \mathbb{E}[ \|\hat{{\bf x}}^{[n+1]} - {\bf x}^s \|_2 | \hat{{\bf x}}^{[n]}] = \mathbb{E}[\|\hat{{\bf x}}_{\hat{B}^{[n+1]}}^{[n+1]} - {\bf x}_{\hat{B}^{[n+1]}}^s \|_2| \hat{{\bf x}}^{[n]}] &\\
    \leq \mathbb{E}[\|\hat{{\bf x}}_{\hat{B}^{[n+1]}}^{[n+1]} - \hat{{\bf a}}_{\hat{B}^{[n+1]}}^{[n+1]} \|_2| \hat{{\bf x}}^{[n]}] + \mathbb{E}[\|\hat{{\bf a}}_{\hat{B}^{[n+1]}}^{[n+1]} -  {\bf x}_{\hat{B}^{[n+1]}}^s& \|_2| \hat{{\bf x}}^{[n]}].
    \end{split}
\end{equation}
where the equality follows from that $\hat{{\bf x}}^{[n+1]} - {\bf x}^s$ is supported over the set $\hat{B}^{[n+1]} = \hat{\Gamma}^{[n+1]} \cup \Gamma^s$.

Recall that $\hat{{\bf x}}_{\hat{B}^{[n+1]}}^{[n+1]}$ is a better s-term approximation to $\hat{{\bf a}}_{\hat{B}^{[n+1]}}^{[n+1]}$ than ${\bf x}_{\hat{B}^{[n+1]}}^{s}$ \big(i.e., $\|\hat{{\bf x}}^{[n+1]} - \hat{{\bf a}}_{\hat{B}^{[n+1]}}^{[n+1]} \|_2\leq \| \hat{{\bf a}}_{\hat{B}^{[n+1]}}^{[n+1]} - {\bf x}^s \|_2$\big). Then 
\begin{equation}\label{starting_bound}
    \mathbb{E}[\|\hat{{\bf x}}^{[n+1]} - {\bf x}^s \|_2| \hat{{\bf x}}^{[n]}] \leq 2\mathbb{E}[ \|\hat{{\bf a}}_{\hat{B}^{[n+1]}}^{[n+1]} - {\bf x}_{\hat{B}^{[n+1]}}^s \|_2 | \hat{{\bf x}}^{[n]}]
\end{equation}
Using triangle inequality, we further have
\begin{equation}\label{nonproof_bound_final}
\begin{split}
    \mathbb{E}[\|\hat{{\bf x}}^{[n+1]} - {\bf x}^s \|_2| \hat{{\bf x}}^{[n]}]
    \leq 2 \big [\mathbb{E}[\|\hat{{\bf a}}_{\hat{B}^{[n+1]}}^{[n+1]} - {{\bf a}}_{\hat{B}^{[n+1]}}^{[n+1]} \|_2  + \|{{\bf a}}_{\hat{B}^{[n+1]}}^{[n+1]} - {{\bf x}}_{\hat{B}^{[n+1]}}^{s} \|_2 | \hat{{\bf x}}^{[n]}]\big ]
    \end{split}
\end{equation}

We now continue with the analysis referring to two terms on the right hand side of (\ref{nonproof_bound_final}) separately.\\

{(a)} Expanding $\hat{{\bf a}}_{\hat{B}^{[n+1]}}^{[n+1]}$ and ${{\bf a}}_{\hat{B}^{[n+1]}}^{[n+1]}$ we have
\begin{equation}\label{quantization_error_terms}
    \begin{split}
      \mathbb{E}[ & \|\hat{{\bf a}}_{\hat{B}^{[n+1]}}^{[n+1]} - {{\bf a}}_{\hat{B}^{[n+1]}}^{[n+1]} \|_2| \hat{{\bf x}}^{[n]}] \\
      &= \mathbb{E}[\|\hat{\mu}^{[n]}Q_1({\bf {\bf \Phi}})_{\hat{B}^{[n+1]}}^T\big(Q_y({\bf y})-Q_2({\bf {\bf \Phi}})\hat{{\bf x}}^{[n]}\big)\\
      & \ \ \ \ \ \ \ \ \ \ \ \  \ \  - {\mu}^{[n]}{\bf {\bf \Phi}}_{\hat{B}^{[n+1]}}^T({\bf y}-{\bf {\bf \Phi}}\hat{{\bf x}}^{[n]}) \|_2| \hat{{\bf x}}^{[n]}]\\
        &=\mathbb{E}[\|\hat{\mu}^{[n]}Q_1({\bf {\bf \Phi}})_{\hat{B}^{[n+1]}}^T\big({\bf {\bf \Phi}} {\bf x}^s +\boldsymbol{\varepsilon}+\boldsymbol{\epsilon}_y-Q_2({\bf {\bf \Phi}})\hat{{\bf x}}^{[n]}\big)\\
         &\ \ \ \ \ \ \ \ \ \ \ \ -{\mu}^{[n]}{\bf {\bf \Phi}}_{\hat{B}^{[n+1]}}^T({\bf {\bf \Phi}}{\bf x}^s +\boldsymbol{\varepsilon}-{\bf {\bf \Phi}}\hat{{\bf x}}^{[n]}) \|_2| \hat{{\bf x}}^{[n]}]\\
         &=\mathbb{E}[\|\hat{\mu}^{[n]}Q_1({\bf {\bf \Phi}})_{\hat{B}^{[n+1]}}^T\big(-Q_2({\bf {\bf \Phi}}) {\bf r}^{[n]} +\boldsymbol{\varepsilon}+\boldsymbol{\epsilon}_y+({\bf {\bf \Phi}}\\
         & \ \ \ \ \ \ -Q_2({\bf {\bf \Phi}})){\bf x}^{s}\big)+{\mu}^{[n]}{\bf {\bf \Phi}}_{\hat{B}^{[n+1]}}^T({\bf {\bf \Phi}}{\bf r}^{[n]} -\boldsymbol{\varepsilon}) \|_2| \hat{{\bf x}}^{[n]}]\\
        &\leq \|\big(\mu^{[n]}{\bf {\bf \Phi}}_{\hat{B}^{[n+1]}}^T{\bf {\bf \Phi}} - \hat{\mu}^{[n]}Q_1({\bf {\bf \Phi}})_{\hat{B}^{[n+1]}}^T Q_2({\bf {\bf \Phi}})\big) {\bf r}^{[n]} \|_2\\
         &\ \ \ \ \ \ \ \ \ \ \ \ \ +\|\big(\mu^{[n]}{\bf {\bf \Phi}}_{\hat{B}^{[n+1]}}^T - \hat{\mu}^{[n]}Q_1({\bf {\bf \Phi}})_{\hat{B}^{[n+1]}}^T \big)\boldsymbol{\varepsilon} \|_2\\
      &\ \ \ \  \ \ \ \ \ \ \ \ \ +\mathbb{E}[\|\hat{\mu}^{[n]}Q_1({\bf {\bf \Phi}})_{\hat{B}^{[n+1]}}^T\boldsymbol{\epsilon}_y \|_2]\\
      &\ \ \ \  \ \ \ \ \ \ \ \ \ +\mathbb{E}[\| \hat{\mu}^{[n]}Q_1(\boldsymbol{\Phi})_{\hat{B}^{[n+1]}}{}^T\big(\boldsymbol{\Phi}-Q_2(\boldsymbol{\Phi})\big){\bf x}^s \|_2].
    \end{split}
\end{equation} 
where we used the expansion ${{\bf r}}^{[n]} = \hat{{\bf x}}^{[n]}-{\bf x}^s$. We further derive the terms governing the above expression in (a.1), (a.2), (a.3) and (a.4). 

(a.1) Since ${\bf r}^{[n]}$ is supported over $\hat{B}^{[n]}$, we clearly have
\begin{equation*}\label{first_term_q_bound}
    \begin{split}
     &\|\big(\mu^{[n]}{\bf {\bf \Phi}}_{\hat{B}^{[n+1]}}^T{\bf {\bf \Phi}} - \hat{\mu}^{[n]}Q_1({\bf {\bf \Phi}})_{\hat{B}^{[n+1]}}^T Q_2({\bf {\bf \Phi}})\big) {\bf r}^{[n]} \|_2\\ 
     & \leq\|\big(\mu^{[n]}{\bf {\bf \Phi}}_{\hat{B}^{[n+1]}}^T{\bf {\bf \Phi}}_{\hat{B}^{[n+1]}} - \hat{\mu}^{[n]}Q_1({\bf {\bf \Phi}})_{\hat{B}^{[n+1]}}^T Q_2({\bf {\bf \Phi}})_{\hat{B}^{[n+1]}}\big)\\
     &\ \ \ \ \ {\bf r}_{\hat{B}^{[n+1]}}^{[n]}\|_2+ \|\big ( \mu^{[n]}{\bf {\bf \Phi}}_{\hat{B}^{[n+1]}}^T{\bf {\bf \Phi}}_{\hat{B}^{[n]}\backslash \hat{B}^{[n+1]}}\\
     &\ \ \ \ \ -\hat{\mu}^{[n]}Q_1({\bf {\bf \Phi}})_{\hat{B}^{[n+1]}}^T Q_2({\bf {\bf \Phi}})_{\hat{B}^{[n]}\backslash \hat{B}^{[n+1]}}\big ) {\bf r}_{\hat{B}^{[n]}\backslash \hat{B}^{[n+1]}}^{[n]} \|_2.
    \end{split}
\end{equation*}
Using the second inequality in Lemma~\ref{lemma_auxiliary_results} we have
{\small
\begin{equation}\label{residual_bound}
\begin{split}
 \|\big(\mu^{[n]}{\bf {\bf \Phi}}_{\hat{B}^{[n+1]}}^T{\bf {\bf \Phi}}_{\hat{B}^{[n+1]}} - \hat{\mu}^{[n]}Q_1({\bf {\bf \Phi}}&)_{\hat{B}^{[n+1]}}^T Q_2({\bf {\bf \Phi}})_{\hat{B}^{[n+1]}}\big) {\bf r}_{\hat{B}^{[n+1]}}^{[n]}\|_2 \\
& \leq (\gamma_{2s} + \hat{\gamma}_{2s})\| {\bf r}_{\hat{B}^{[n+1]}}^{[n]}\|_2.
 \end{split}
\end{equation}
}
Let now $\hat{B}^{[n+1]}$ be split into two disjoint sets $\Gamma_1$ and $\Gamma_2$, where $\Gamma_1 \cap \Gamma_2 = \emptyset$ and $|\Gamma_1|, |\Gamma_2| \leq s$. By the third inequality in Lemma~\ref{lemma_auxiliary_results}, we have
\begin{equation}\label{residual_full_precision}
\begin{split}
&\|\big ( \mu^{[n]}{\bf \Phi}_{\hat{B}^{[n+1]}}{}^T{\bf \Phi}_{\hat{B}^{[n]}\backslash \hat{B}^{[n+1]}}\\
&\ \ \ \ - \hat{\mu}^{[n]}Q_1({\bf \Phi})_{\hat{B}^{[n+1]}}{}^T Q_2({\bf \Phi})_{\hat{B}^{[n]}\backslash \hat{B}^{[n+1]}}\big ) {\bf r}^{[n]}_{\hat{B}^{[n]}\backslash \hat{B}^{[n+1]}} \|_2 \\
&\leq \bigg(\|\big ( \mu^{[n]}{\bf \Phi}_{\Gamma_1}{}^T{\bf \Phi}_{\hat{B}^{[n]}\backslash \hat{B}^{[n+1]}}\\
& \ \ \ \ -\hat{\mu}^{[n]}Q_1({\bf \Phi})_{\Gamma_1}{}^T Q_2({\bf \Phi})_{\hat{B}^{[n]}\backslash \hat{B}^{[n+1]}}\big ) {\bf r}^{[n]}_{\hat{B}^{[n]}\backslash \hat{B}^{[n+1]}}\|_2^2\\
&\ \ \ \ + \|\big ( \mu^{[n]}{\bf \Phi}_{\Gamma_2}{}^T{\bf \Phi}_{\hat{B}^{[n]}\backslash \hat{B}^{[n+1]}}\\
& \ \ \ \ - \hat{\mu}^{[n]}Q_1({\bf \Phi})_{\Gamma_2}{}^T Q_2({\bf \Phi})_{\hat{B}^{[n]}\backslash \hat{B}^{[n+1]}}\big ) {\bf r}^{[n]}_{\hat{B}^{[n]}\backslash \hat{B}^{[n+1]}}\|_2^2\bigg)^{\frac{1}{2}}\\
& \leq \sqrt{2}(\gamma_{2s}+  \hat{\gamma}_{2s})\| {\bf r}^{[n]}_{\hat{B}^{[n]}\backslash \hat{B}^{[n+1]}}\|_2.
\end{split}
\end{equation}
Combining (\ref{residual_bound}) and (\ref{residual_full_precision}),
\begin{equation}\label{a1}
    \begin{split}
     &\|\big(\mu^{[n]}{\bf \Phi}_{\hat{B}^{[n+1]}}^T{\bf \Phi} - \hat{\mu}^{[n]}Q_1({\bf \Phi})_{\hat{B}^{[n+1]}}^T Q_2({\bf \Phi})\big) {\bf r}^{[n]} \|_2\\
     &= (\gamma_{2s}+ \hat{\gamma}_{2s})\|{\bf r}^{[n]}_{\hat{B}^{[n+1]}}\|_2 + \sqrt{2}(\gamma_{2s}+ \hat{\gamma}_{2s})\|{\bf r}^{[n]}_{\hat{B}^{[n]}\backslash \hat{B}^{[n+1]}} \|_2\\
     &\leq 2 (\gamma_{2s}+ \hat{\gamma}_{2s})\|{\bf r}^{[n]} \|_2
    \end{split}
\end{equation}
where the last inequality follows from the fact that ${\bf r}_{\hat{B}^{[n+1]}}^{[n]}$ and ${\bf r}_{\hat{B}^{[n]}\backslash \hat{B}^{[n+1]}}$ are orthogonal.

(a.2) Expanding the second term in (\ref{quantization_error_terms})
\begin{equation}
    \begin{split}
        &\|(\mu^{[n]}{\bf \Phi}_{\hat{B}^{[n+1]}} - \hat{\mu}^{[n]}Q_1({\bf \Phi})_{\hat{B}^{[n+1]}}^T)\boldsymbol{\varepsilon} \|_2 \\ 
        &\leq\|(\mu^{[n]}{\bf \Phi}_{\hat{B}^{[n+1]}} - \hat{\mu}^{[n]}Q_1({\bf \Phi})_{\hat{B}^{[n+1]}}^T){\bf e} \|_2\\
        & \ \  \ \ + \|(\mu^{[n]}{\bf \Phi}_{\hat{B}^{[n+1]}} - \hat{\mu}^{[n]}Q_1({\bf \Phi})_{\hat{B}^{[n+1]}}^T){\bf \Phi} ({\bf x}-{\bf x}^s) \|_2.
    \end{split}
\end{equation}

Using (\ref{residual_x_bound}) and Lemma~\ref{lemma_auxiliary_results} we have
\begin{equation}\label{second_term_quantization_error}
    \begin{split}
    &\|(\mu^{[n]}{\bf \Phi}_{\hat{B}^{[n+1]}} - \hat{\mu}^{[n]}Q_1({\bf \Phi})_{\hat{B}^{[n+1]}}^T){\bf e} \|_2\\
    &\leq \max \big (  (1+\gamma_{2s})/\beta_{2s}, (1+\hat{\gamma}_{2s})/\hat{\beta}_{2s}\big ) \|{\bf e} \|_2\\
    & \|(\mu^{[n]}{\bf \Phi}_{\hat{B}^{[n+1]}} - \hat{\mu}^{[n]}Q_1({\bf \Phi})_{\hat{B}^{[n+1]}}^T){\bf \Phi} ({\bf x}-{\bf x}^s) \|_2\\
        &\leq \big ( \| (\mu^{[n]}{\bf \Phi}_{\Gamma_1} - \hat{\mu}^{[n]}Q_1({\bf \Phi})_{\Gamma_1}^T){\bf \Phi} ({\bf x}-{\bf x}^s) \|^2_2\\
        & \ \ \ \ \ \ \ + \|(\mu^{[n]}{\bf \Phi}_{\Gamma_2} - \hat{\mu}^{[n]}Q_1({\bf \Phi})_{\Gamma_2}^T){\bf \Phi} ({\bf x}-{\bf x}^s) \|^2_2\big )^{1/2}\\
        & \leq\sqrt{2}(\hat{\gamma}_{2s}+ \hat{\gamma}_{2s}) \bigg [ \|{\bf x}-{\bf x}^s \|_2 + \frac{\|{\bf x}-{\bf x}^s \|_1}{\sqrt{s}}\bigg].
    \end{split}
\end{equation}
Combining results obtained in~(\ref{second_term_quantization_error})
\begin{equation}\label{a2}
\begin{split}
    &\|(\mu^{[n]}{\bf \Phi}_{\hat{B}^{[n+1]}} - \hat{\mu}^{[n]}Q_1({\bf \Phi})_{\hat{B}^{[n+1]}}^T)\boldsymbol{\bf \varepsilon} \|_2 \\
    &\leq  \ \max \big (  (1+\gamma_{2s})/\beta_{2s}, (1+\hat{\gamma}_{2s})/\hat{\beta}_{2s}\big )  \|{\bf e} \|_2\\
    & \ \ \  + \sqrt{2}({\gamma}_{2s}+ \hat{\gamma}_{2s})\bigg  [ \|{\bf x}-{\bf x}^s \|_2 + \frac{\|{\bf x}-{\bf x}^s \|_1}{\sqrt{s}}\bigg].
\end{split}
\end{equation}
(a.3) The third term of (\ref{quantization_error_terms})
\begin{equation}\label{a3}
    \begin{split}
         \mathbb{E}[\|\hat{\mu}^{[n+1]}Q_1({\bf \Phi})^T _{\hat{B}^{[n+1]}}\boldsymbol{\epsilon}_y \|_2]& \ \stackrel{(1)}{\leq} \  \frac{(1+\hat{\gamma}_{2s})}{\hat{\beta}_{2s}}\mathbb{E}[\|\boldsymbol{\epsilon}_y \|_2]\\
         &\ \stackrel{(2)}{\leq}\frac{(1+\hat{\gamma}_{2s})c_y\sqrt{M}}{\hat{\beta}_{2s}2^{b_{\bf y}-1}}.
    \end{split}
\end{equation}
where the inequalities follows from (1) (\ref{Eq:rip_lowprecision}) together with Remark~\ref{Remark:mu_bounds}, and (2) Lemma~\ref{lemma_on_quantized_vector}.

(a.4) Combining with (\ref{Eq:rip_lowprecision}), Remark~\ref{Remark:mu_bounds}, {\it Cauchy-Bunyakovsky-Schwarz}, {\it Jensen} inequalities and the similar discussion above
\begin{equation}\label{a4}
\begin{split}
    &\mathbb{E}[\| \hat{\mu}^{[n]}Q_1({\bf \Phi})_{\hat{B}^{[n+1]}}^T\big({\bf {\bf \Phi}}-Q_2({\bf {\bf \Phi}})\big){\bf x}^s \|_2]\\
    & \leq \ \frac{(1+\hat{\gamma}_{2s})}{\hat{\beta}_{2s}}\mathbb{E}[\|\big({\bf {\bf \Phi}}-Q_2({\bf {\bf \Phi}})\big){\bf x}^s \|_2]\\
    & \leq \ \frac{(1+\hat{\gamma}_{2s})}{\hat{\beta}_{2s}}\sqrt{
    \sum_{i}^M \sum_{j}^N\mathbb{E}[({\bf \Phi}_{i, j}-Q_2({\bf \Phi}_{i, j}){\bf x}_j^s)^2]}\\
    &  =  \frac{(1+\hat{\gamma}_{2s})c_{\boldsymbol{\Phi}}\sqrt{M}}{\hat{\beta}_{2s}2^{b_{\bf \Phi}-1}}\|{\bf x}^s \|_2.
\end{split}
\end{equation}

(b) Finally, we bound the second term on the right hand side of (\ref{nonproof_bound_final}) as follows.
\begin{equation}\label{last_term_main_theorem}
    \begin{split}
       &\|{{\bf a}}_{\hat{B}^{[n+1]}}^{[n+1]} - {{\bf x}}_{\hat{B}^{[n+1]}}^{s} \|_2\\
       &= \|{\hat{\bf x}}_{\hat{B}^{[n+1]}}^{[n]} + \mu^{[n]}{\bf \Phi}_{\hat{B}^{[n+1]}}^T({\bf y}-{\bf \Phi} \hat{{\bf x}}^{[n]}) - {{\bf x}}_{\hat{B}^{[n+1]}}^{s} \|_2\\
        &= \|{\hat{\bf x}}_{\hat{B}^{[n+1]}}^{[n]} + \mu^{[n]}{\bf \Phi}_{\hat{B}^{[n+1]}}^T({\bf \Phi} {\bf x}^s + \boldsymbol{\varepsilon}-{\bf \Phi} \hat{\bf x}^{[n]}) - {{\bf x}}_{\hat{B}^{[n+1]}}^{s} \|_2\\
         &= \|{{\bf r}}_{\hat{B}^{[n+1]}}^{[n]} - \mu^{[n]}{\bf \Phi}_{\hat{B}^{[n+1]}}^T({\bf \Phi} {\bf r}^{[n]}  -\boldsymbol{\varepsilon}) \|_2\\
         &= \|{{\bf r}}_{\hat{B}^{[n+1]}}^{[n]} - \mu^{[n]}{\bf \Phi}_{\hat{B}^{[n+1]}}^T({\bf \Phi}_{\hat{B}^{[n+1]}}{\bf r}^{[n]}_{\hat{B}^{[n+1]}}\\
         &\ \ \ \ + {\bf \Phi}_{\hat{B}^{[n]}\backslash \hat{B}^{[n+1]}}{\bf r}^{[n]}_{\hat{B}^{[n]}\backslash \hat{B}^{[n+1]}}  - \boldsymbol{\varepsilon}) \|_2\\
         &\leq \|({\boldsymbol{\rm I}} - \mu^{[n]}{\bf \Phi}_{\hat{B}^{[n+1]}}^T{\bf \Phi}_{\hat{B}^{[n+1]}}){\bf r}^{[n]}_{\hat{B}^{[n+1]}}\|_2\\
         &\ \ \ \ + \|\mu^{[n]}{\bf \Phi}_{\hat{B}^{[n+1]}}^T{\bf \Phi}_{\hat{B}^{[n]}\backslash \hat{B}^{[n+1]}}{\bf r}^{[n]}_{\hat{B}^{[n]}\backslash \hat{B}^{[n+1]}}\|_2\\
         &\ \ \ \ + \|\mu^{[n]}{\bf \Phi}_{\hat{B}^{[n+1]}}^T\boldsymbol{\varepsilon}\|_2.
    \end{split}
\end{equation}

It can be verified by using (\ref{Eq:rip}), Remark~\ref{Remark:mu_bounds} and (\ref{residual_x_bound}) that
\begin{equation}\label{last_term_main_theorem2}
    \begin{split}
        &\|({\boldsymbol{\rm I}} - \mu^{[n]}{\bf \Phi}_{\hat{B}^{[n+1]}}{}^T{\bf \Phi}_{\hat{B}^{[n+1]}}){\bf r}^{[n]}_{\hat{B}^{[n+1]}}\|_2 \overset{(1)}{\leq} \gamma_{2s}\| {\bf r}^{[n]}_{\hat{B}^{[n+1]}}\|_2\\
        &\|\mu^{[n]}{\bf \Phi}_{\hat{B}^{[n+1]}}^T{\bf \Phi}_{\hat{B}^{[n]}\backslash \hat{B}^{[n+1]}}{\bf r}^{[n]}_{\hat{B}^{[n]}\backslash \hat{B}^{[n+1]}}\|_2\\
        &{\leq} \big (\|\mu^{[n]}{\bf \Phi}_{\Gamma_1}^T{\bf \Phi}_{\hat{B}^{[n]}\backslash \hat{B}^{[n+1]}}{\bf r}^{[n]}_{\hat{B}^{[n]}\backslash \hat{B}^{[n+1]}}\|_2^2\\
        &\ \  + \|\mu^{[n]}{\bf \Phi}_{\Gamma_2}^T{\bf \Phi}_{\hat{B}^{[n]}\backslash \hat{B}^{[n+1]}}{\bf r}^{[n]}_{\hat{B}^{[n]}\backslash \hat{B}^{[n+1]}}\|_2^2\big )^{1/2}\\
        &\overset{(2)}{\leq}\sqrt{2}\gamma_{2s}\| {\bf r}^{[n]}_{\hat{B}^{[n]}\backslash \hat{B}^{[n+1]}}\|_2\\
        &\|\mu^{[n]}{\bf \Phi}_{\hat{B}^{[n+1]}}^T\varepsilon\|_2\\ &\overset{(3)}{\leq} \frac{1+\gamma_{2s}}{\beta_{2s}}\| {\bf e}\|_2 + \sqrt{2}(1+\gamma_{2s}) \Big [ \|{\bf x} - {\bf x}^s\|_2 - \frac{\|{\bf x} - {\bf x}^s\|_1}{\sqrt{s}}\Big].
    \end{split}
\end{equation}

By the orthogonality between ${\bf r}_{\hat{B}^{[n+1]}}^{[n]}$ and ${\bf r}_{\hat{B}^{[n]} \backslash \hat{B}^{[n+1]}}^{[n]}$, (\ref{last_term_main_theorem}) can further be simplified to
\begin{equation}\label{b}
    \begin{split}
        &\|{{\bf a}}_{\hat{B}^{[n+1]}}^{[n+1]} - {{\bf x}}_{\hat{B}^{[n+1]}}^{s} \|_2 |{\bf x}^{[n]}\\
        &\leq 2 \gamma_{2s} \| {\bf r}^{[n]}\|_2 + \frac{1+\gamma_{2s}}{\beta_{2s}}\| {\bf e}\|_2\\
        & \ \ \ + \sqrt{2}(1+\gamma_{2s}) \Big [ \|{\bf x} - {\bf x}^s\|_2 - \frac{\|{\bf x} - {\bf x}^s\|_1}{\sqrt{s}}\Big].
    \end{split}
\end{equation}

Substituting (\ref{a1}), (\ref{a2}), (\ref{a3}), (\ref{a4}) and (\ref{b}) into (\ref{starting_bound}), the norm of recovery error is given by
\begin{equation}\label{final_error_residual}
    \begin{split}
    &\mathbb{E}[\| {\bf r}^{[n+1]}\|_2 |{\bf r}^{[n]}]\\
    &\leq 12\max(\gamma_{2s},\hat{\gamma}_{2s})\|{\bf r}^{[n]}\|_2+4\max \big ( \frac{1+\gamma_{2s}}{\beta_{2s}},\frac{1+\hat{\gamma}_{2s}}{\hat{\beta}_{2s}} \big)\|{\bf e}\|_2\\
    & \ \ \ \  +2\sqrt{2}(3\max({\gamma}_{2s},\hat{\gamma}_{2s})+1)\bigg  [ \|{\bf x}-{\bf x}^s \|_2 + \frac{\|{\bf x}-{\bf x}^s \|_1}{\sqrt{s}}\bigg]\\
        & \ \ \ \ +2\frac{(1+\hat{\gamma}_{2s})\sqrt{M}}{\hat{\beta}_{2s}}\big ( \frac{c_{\bf \Phi}\|{\bf x}^s \|_2}{^{2^{b_{\bf \Phi}-1}}} + \frac{c_{\bf y}}{2^{b_{\bf y}-1}} \big )
    \end{split}
\end{equation}
Let ${\gamma_{2s}, \hat{\gamma}_{2s}} \leq t$. For $t\leq 1/24$, we have

\begin{equation*}
\begin{split}
        \mathbb{E}&[\|\hat{{\bf x}}^{[n+1]}-{\bf x}^s\|_2|\hat{{\bf x}}^{[0]}={\bf 0}] \leq 2^{-n}\|{\bf x}^{s}\|_2\\
        & \ \ \ \ + \frac{8.4}{\min(\beta_{2s},\hat{\beta}_{2s})}\|{\bf e}\|_2\\
        & \ \ \ \ +6.4\bigg  [ \|{\bf x}-{\bf x}^s \|_2 + \frac{\|{\bf x}-{\bf x}^s \|_1}{\sqrt{s}}\bigg]\\
        & \ \ \ \ +\frac{4.2\sqrt{M}}{\hat{\beta}_{2s}}\big ( \frac{c_{\bf \Phi}\|{\bf x}^s \|_2}{^{2^{b_{\bf \Phi}-1}}} + \frac{c_{\bf y}}{2^{b_{\bf y}-1}} \big )
\end{split}
\end{equation*}

and using the following notation:
\begin{equation*}
\begin{split}
      \epsilon_s & := \|{\bf x}-{\bf x}^s \|_2 + \frac{\|{\bf x}-{\bf x}^s \|_1}{\sqrt{s}} + \frac{1}{\min(\beta_{2s},\hat{\beta}_{2s})}||{\bf e}||_2 \\
     \epsilon_q & :=   \frac{\sqrt{M}}{\hat{\beta}_{2s}} \bigg ( \frac{\|{c_{\bf \Phi}\bf x}^s \|_2}{^{2^{b_{\bf \Phi}-1}}} + \frac{c_{\bf y}}{2^{b_{\bf y}-1}}\bigg )
\end{split}
\end{equation*}
we finally have
\begin{equation*}\label{general bound}
        \mathbb{E}[\|\hat{{\bf x}}^{[n+1]}-{\bf x}^s\|_2|\hat{{\bf x}}^{[0]}={\bf 0}] \leq 2^{-n}\|{\bf x}^{s}\|_2        +9\epsilon_s+ 4.5\epsilon_q.
\end{equation*}
\subsection{Proof of Lemma 1}\label{Sec: proof of lemma}
 Assume that ${\bf \Phi}_{\Gamma}$ has the singular values confined in $[\alpha_{|\Gamma|}, \beta_{|\Gamma|}]$. Through the perturbation of singular values of a matrix upon corruption of entries with noise, it is shown that Bernoulli noise, corrupting the entries of the matrix independently, lifts up the singular values of the matrix, and at most by $\sigma_{\max}\sqrt{|\Gamma|}$ where $\sigma_{\max}$ is the maximum of the noise standard deviations~\cite{steawart1990perturbation, steawart2006perturbation, vaccaro1987perturbation}. Therefore, singular values of $\hat{\bf \Phi}_{\Gamma}$ is in $[\alpha_{|\Gamma|}, \beta_{|\Gamma|}+\sigma_{\max}\sqrt{|\Gamma|}]$. Moreover, we previously showed that the variance of the quantization noise is at most $1/2^{b-1}$, hence we have $\sigma_{\max}=1/2^{b-1}$. Thus, $\hat{\gamma}_{|\Gamma|}$ satisfies
 \begin{equation*}
     \hat{\gamma}_{|\Gamma|} \leq {\gamma}_{|\Gamma|}+\frac{\sqrt{|\Gamma|}}{2^{b-1}\alpha_{|\Gamma|}}
 \end{equation*}
 
The above equation guarantees that whenever $\gamma_{|\Gamma|} +\epsilon\leq 1/24$, for some $\epsilon \geq\frac{\sqrt{|\Gamma|}}{2^{b-1}\alpha_{|\Gamma|}}$, $\hat{\gamma}_{|\Gamma|}$ is guaranteed to be lower than $1/24$.

\bibliographystyle{IEEEtran}
\bibliography{IEEEabrv,Bibliography}

\end{document}